%% file: arxiv_main.tex
\documentclass{article} % For LaTeX2e
\usepackage{arxiv,times}

\input{math_commands.tex}

\usepackage[utf8]{inputenc} % allow utf-8 input
\usepackage[T1]{fontenc}    % use 8-bit T1 fonts
\usepackage{hyperref}       % hyperlinks
\usepackage{url}            % simple URL typesetting
\usepackage{booktabs}       % professional-quality tables
\usepackage{amsfonts}       % blackboard math symbols
\usepackage{nicefrac}       % compact symbols for 1/2, etc.
\usepackage{microtype}      % microtypography
\usepackage{lipsum}		% Can be removed after putting your text content
\usepackage{graphicx}
\usepackage{natbib}
\usepackage{doi}

\usepackage{array}
\usepackage{graphicx}
\usepackage{booktabs}
\usepackage{multirow}
\usepackage[most]{tcolorbox}
\usepackage{enumitem}
\setlist[itemize]{leftmargin=2em}
\usepackage{inconsolata}
\usepackage{wrapfig}
\usepackage{subcaption}
\usepackage{amsthm}
\newtheorem{theorem}{Theorem}

\definecolor{darkgreen}{RGB}{1,80,3}
\usepackage[ruled,lined,linesnumbered]{algorithm2e}

\SetCommentSty{mycommfont}

\newcommand{\name}{3D Scaling}

\usepackage{listings}
\usepackage{xcolor}
\usepackage{tikz}
\usepackage{fancyhdr}
\title{Extending Test-Time Scaling: A 3D Perspective with Context, Batch, and Turn}
\author{
\textbf{Chao Yu$^{1*}$, Qixin Tan$^{1*}$, Jiaxuan Gao$^{1}$, Shi Yu$^{1}$, Hong Lu$^{1}$,} \\
\textbf{Xinting Yang$^{1}$, Zelai Xu$^{1}$, Yu Wang$^{1\dagger}$, Yi Wu$^{1\dagger}$, Eugene Vinitsky$^{2\dagger}$} \\
$^{1}$Tsinghua University \quad
$^{2}$New York University \\
\texttt{zoeyuchao@gmail.com, yu-wang@tsinghua.edu.cn, jxwuyi@gmail.com, vinitsky.eugene@gmail.com}
}
\date{}

\renewcommand{\headeright}{}
\renewcommand{\undertitle}{}

\begin{document}
\maketitle
\makeatletter
\begingroup
\renewcommand\thefootnote{}
\renewcommand\@makefntext[1]{\noindent#1}
\footnotetext{
$^{*}$ Chao Yu and Qixin Tan contributed equally to this work.\\
$^{\dagger}$ Corresponding author.
}
\addtocounter{footnote}{0}
\endgroup
\makeatother

\input{00_abstract}

\section{Introduction}\label{sec:intro}
\input{01_intro}

\section{Related Work}\label{sec:related}
\input{02_related}

\section{Formulation of Test-Time Scaling}\label{sec:prelimary}

% \yw{We should formally define each dimension of test time scaling. Mostly copy from the optimal reasoning efficiency paper}

% \yw{One suggestion is that we can define a function $TTS(C,B,T)$ to denote the compute used by context $C$, batch size $B$ and turn $T$}

\paragraph{LLM Reasoning.} In this work, we focus on LLM reasoning. Given a question $x\in \mathcal X$, the goal is to derive a correct step-by-step solution $y \in \mathcal{Y}$. We assume the existence of a ground-truth verifier $ \mathcal{R}(x,y) $ that evaluates the correctness of a solution $y$ for a question $x$. 
\iffalse
Formally, the verifier is defined as a function
$$
\mathcal{R}: \mathcal{X} \times \mathcal{Y} \rightarrow \mathcal{S},
$$
where $ \mathcal{S} $ denotes the space of correctness signals 
(e.g., scalar scores, binary indicators, or structured feedback). 
\fi
This verifier $\mathcal R$ could have different implementations depending on the specific task in practice. 
For example, in mathematical reasoning tasks where the goal is to derive a single numerical answer, the verifier could return a 0-1 score indicating whether the answer in the solution $y$ matches the ground-truth answer. In coding tasks, the score is determined by the set of unit tests passed by the submitted code in solution $y$.

\iffalse
For example, in mathematical reasoning tasks where the goal is to derive a single numerical answer, this verifier can be instantiated as
$$
\mathcal{R}(x, y) = 
\begin{cases}
1, & \text{if the final answer extracted from } y \text{ matches the ground-truth answer of } x,\\
0, & \text{otherwise.}
\end{cases}
$$
\fi
% In coding tasks, the score is determined by the set of tests passed by the submitted code in solution $y$. 

An LLM $\pi_\theta$ is a policy parameterized by $\theta$. Given an input question $x$, the LLM auto-regressively generates an array of tokens one by one. For a distribution of questions $\mathcal D$, the expected score of an LLM policy $\pi_\theta$ given a question $x$ is defined as,
$$
J(\mathcal D,\pi_\theta)=\mathbb E_{x\sim\mathcal D,y\sim\pi_\theta(\cdot|x)}[\mathcal R(x, y)].
$$

\paragraph{Test-Time Scaling.} Test-time scaling approaches aim to achieve a better score through spending more test-time compute. For instance, context scaling allows the LLM to generate longer responses to conduct in-depth exploration. The efficacy of any test-time scaling method must be evaluated along two key aspects: the expected score and the computational cost. In this work, we quantify computational cost using the theoretical maximum number of tokens generated throughout the inference process.

\begin{figure*}[htbp]
    \centering
    % \vspace{2mm}
    \includegraphics[width=\linewidth]{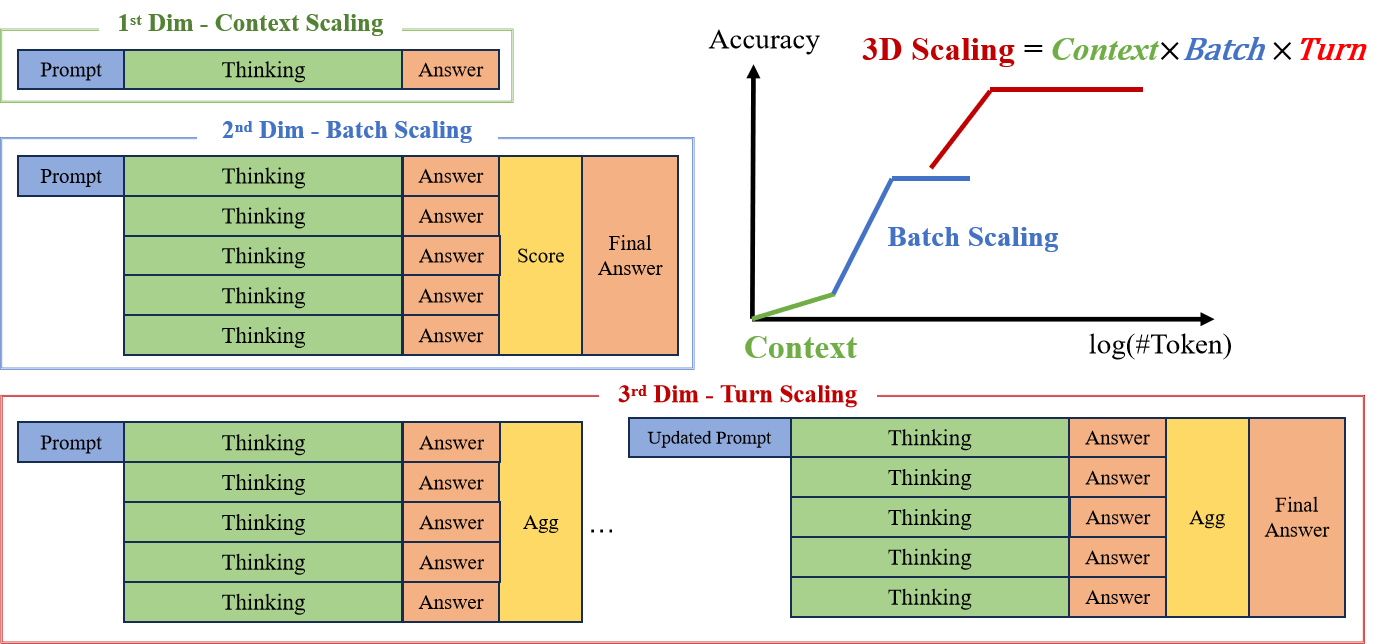}
    \caption{Illustration of Test-time Scaling across three dimensions: context, batch, and turn.
    % \eugene{Watch out in the figure on the right, axes that don't really have any scale are very frustrating. Putting this comment back; unscaled axes are something one would do in a company promotion not a paper. I'd rather remove this entirely.}
    }
    \vspace{-3mm}
    \label{fig: test-time-scaling}
\end{figure*}

\iffalse
\paragraph{Test-Time Scaling.} Test-time scaling approaches aim to achieve a better score than direct generation with the LLM policy through increasing test-time compute. A test-time scaling algorithm can be represented as a function $f$ that takes an input question $x$, an LLM policy $\pi_\theta$, and a set of hyperparameters $H$ as input, and outputs a response $y$. We represent a test-time scaling algorithm as,
\begin{align}
y\sim f(\cdot|x,\pi_\theta, H),
\end{align}

The expected score of the test-time scaling algorithm could be further defined as $J(\mathcal D, f(\cdot, \pi_\theta, H))=\mathbb E_{x\sim\mathcal D,y\sim f(\cdot|x,\pi_\theta, H)}[\mathcal R(x, y)]$.
\fi

\subsection{Test-Time Scaling with Context, Batch, and Turn}

%\paragraph{Context Scaling.} We first consider the context dimension. In context scaling, we follow the most straightforward approach of controlling the reasoning process with a maximum token budget $C$. Specifically, when the LLM generates a response $y\sim\pi_\theta(\cdot|x)$ exceeding $C$, we directly truncate the response. Therefore, the expected score of context scaling under a context length $C$ is $J_{\texttt{context}}(\mathcal D, \pi_\theta, C)=\mathbb E_{x\sim\mathcal D,y\sim\pi_\theta(\cdot|x)}[\mathcal R(x, y_{:C})]$

\subsubsection{Context Scaling}

%\paragraph{Context Scaling.}
We first consider scaling along the context dimension that directly controls the response length. In context scaling, the model is explicitly guided to generate a response under an allocated token budget $C$. The reasoning process is conditioned on the prompt $x$ and continues until it reaches an end-of-sequence token or the maximum context length $C$. 
The expected score of context scaling under a context budget $C$ is defined as the expected reward of the full generated response \( y \) within the  context length,
\begin{align}
J_{\texttt{context}}(\mathcal D, \pi_\theta, C) = \mathbb E_{x\sim\mathcal D}[\mathcal R(x, y)], \quad \text{where } y \sim \pi_\theta(\cdot|x) \text{ and } |y| \leq C
\label{eq;context}
\end{align}

% Increasing $C$ allows the model to produce more reasoning steps, thereby scaling the problem-solving capacity. 

\subsubsection{Batch Scaling}
%\paragraph{Batch Scaling.} To scale test-time compute along the batch dimension, notable examples of batch scaling approaches include Majority Voting and Best-of-N~\citep{wang2023selfconsistencyimproveschainthought,snell2024scalingllmtesttimecompute,cobbe2021training}. Given the size of the batch dimension $B$, i.e., the number of responses generated in parallel, and the context length $C$, $B$ responses $y_1,y_2,\cdots, y_B$ are first generated under a token budget of $C$. Then, a final solution is derived by selecting the best response through a scoring function $\text{Score}(y)$, i.e. $y_{\text{final}}=\arg\max_i\text{Score}(y_{i})$. In our implementation, we extend the standard majority-vote procedure intomploy a two-stage selection strategy: we first perform majority voting on the final answers of the responses; if this process results in multiple candidate answers with the same highest vote count, we then perform a Best-of-$N$ selection among the set of responses corresponding to these tied answers. In standard cases where the goal is to derive a single numerical answer, the scoring function could be implemented by counting the occurrence of each candidate answer in the $B$ responses. In coding tasks where it is infeasible to apply voting, we select the solution with the highest task score out of the $B$ responses. In this paper, we refer to the two-stage majority-vote-with-tie-breaker procedure as \textbf{Batch Scaling (Vote)} and the Best-of-N strategy as \textbf{Batch Scaling (Best-of-N)}.

%\paragraph{Batch Scaling.}
In this section, we investigate another dimension, the batch dimension, that involves generating multiple independent candidate responses and selecting the response that is most likely correct. Generating multiple candidates allows exploring diverse solution paths. Given a batch size $B$ and a per-response context budget $C$, we first generate $B$ responses independently,

\begin{align}
\{y_1, y_2, \cdots, y_B\} \quad \text{where} \quad y_i \sim \pi_\theta(\cdot|x) \quad \text{and} \quad |y_i| \leq C \quad \forall i \in [1, B]
\end{align}

After the set of $B$ responses are generated, the core of batch scaling is an aggregation function \(\mathcal{A}\) that maps the set of \(B\) responses to a single response \(y_{\text{final}}\), i.e. $ y_{\text{final}} = \mathcal{A}(\{y_1, \cdots, y_B\})$. The expected score of batch scaling is,
\begin{align}
J_{\texttt{batch}}(\mathcal D, \pi_\theta, B, C) = \mathbb E_{x\sim\mathcal D}[\mathcal R(x, y_{\text{final}})], \quad \text{where} \quad y_{\text{final}} = \mathcal{A}(\{y_1, \cdots, y_B\})
\end{align}

The choice of the aggregation function \(\mathcal{A}\) is critical and could lead to different practical implementations.

\subparagraph{Choices of Aggregation Function \(\mathcal{A}\).}
The choice of \(\mathcal{A}\) depends on the task structure and the available information. We define two primary strategies,

1.  \textbf{Batch Scaling (Best-of-N):}  
For the Best-of-N strategy, we employ two practical implementations that are suitable for different task configurations,

\begin{itemize}
    \item 
    \textbf{Scoring-based Best-of-N.}  
    % This strategy is used when directly evaluating response quality is difficult 
    % (for example, in reinforcement learning tasks where the reward can only be computed after using the response to train an agent in the environment). 
    We define $S_{\text{task}}(y)$ as a task-specific scoring function that evaluates the response with a scalar value. For example, in programming tasks, the pass rate of unit tests of the code can be directly used as the scoring function.
    The aggregation function then selects the response with the highest score,
    \begin{align}
        \mathcal{A}_{\text{Scoring-based BoN}}(\{y_1, \ldots, y_B\}) 
        = \arg\max_{y \in \{y_1, \ldots, y_B\}} S_{\text{task}}(y)
    \end{align}
    % To use $\mathcal{A}_{\text{Scoring-based BoN}}$, a scoring function should be readily defined. 
    % Under this framework, outputs with higher scores are assigned a greater likelihood of achieving a larger value of $\mathcal{R}(x, y_{\text{final}})$.

    \item 
    
    \textbf{LLM-based Best-of-N.}  
    The LLM is used to directly select the best response among all candidates during aggregation. %, the LLM is prompted to choose the most appropriate response among all generated responses. 
    %Let the chosen response be $y^*$. 
    The aggregation function is expressed as,
    \begin{align}
        \mathcal{A}_{\text{LLM-based BoN}}(\{y_1, \ldots, y_B\}) 
        \sim \pi_\theta(\cdot|[\text{Select the best among }y_1, \ldots, y_B])
    \end{align}
    $\mathcal{A}_{\text{LLM-based BoN}}$ is used when a scoring function for directly evaluating responses is infeasible. For example, in mathematical proof problems, it is often difficult to evaluate the correctness of intermediate reasoning steps of a solution. In such cases, an LLM can perform the evaluation for the whole solution.
    % (for example, in reinforcement learning tasks where the reward can only be computed after using the response to train an agent in the environment). 
\end{itemize}

%This method is powerful but requires a reliable and computationally feasible scoring mechanism.

2.  \textbf{Batch Scaling (Vote):} We also consider majority voting as a representative aggregation strategy when the final answer can be easily extracted from a response \(y_i\) via a deterministic function \(\text{Extract}(y_i)\).  Let $a_i$ be the answer extracted from response $y_i$, i.e. \(a_i = \text{Extract}(y_i)\). The voting function $\mathcal{A}_{\text{Vote}}$ finds the most common answer \(a_{\text{maj}}\) that have the highest frequency, \(a_{\text{maj}} = \arg\max_{a} |\{i : a_i = a\}|\). Formally,
\begin{align}
        \mathcal{A}_{\text{Vote}}(\{y_1, \cdots, y_B\}) = \arg\max_{a} |\{i : a_i = a\}|
\end{align}

Note that \text{Batch Scaling (Vote)} is not applicable when the final answer is complex and equivalence between answers could not be judged efficiently, such as coding tasks. Also, in special problems where the final correctness should be judged based on not only the final answer but also intermediate steps, we employ the LLM to select the best response among all the responses arriving at the most common answer $a_{\text{maj}}$. 

\iffalse
In cases where the final correctness should be judged based on the full solution instead of the final answer only, a tie-breaking procedure is invoked after final answer voting. In our implementation, we fall back to the Best-of-N strategy over the subset of responses that produced the most common answers,
\begin{align}
        \mathcal{A}_{\text{Final Vote}}(\{y_1, \cdots, y_B\}) = 
        \mathcal{A}_{\text{Best-of-N}}(\{y_i : a_i=a_{\text{maj}}\}) \text{  where}\;\;a_{\text{maj}} = \arg\max_{a} |\{i : a_i = a\}|
\end{align}
where LLM-based Best-of-N strategy is employed to select the response that is most likely correct.
\fi

% In this paper, we refer to the two-stage majority-vote-with-tie-breaker procedure as \textbf{Batch Scaling (Vote)} and the direct selection strategy as \textbf{Batch Scaling (Best-of-N)}.

\subsubsection{Turn Scaling}

% We formalize scaling along the turn dimension, which we use the LLM to generate multiple responses sequencially. We put the response generated the last time as the input again into the LLM, then generate again and again for T times. 

% Where the response can be formulated as:
% \begin{align}
%     y^1\sim \pi_\theta(\cdot|[\text{Problem Statement}])
% \end{align}
% \begin{align}
%     y^{t+1} \sim \pi_\theta(\cdot|[\text{Problem Statement, $y^t$}]) 
% \end{align}

Besides context and batch scaling, we also investigate scaling along the turn dimension, that allows the LLM to revise its solution sequentially for $T$ turns. At each turn $t$, the LLM generates a new solution $y^t$ under a context length $C$ based on the solution $y^{t-1}$ from the last turn $t-1$. Formally, the expected score of turn-scaling is,
\begin{align}
    J_{\text{turn}}(\mathcal D,\pi_\theta, T, C)=&\mathbb E_{x\sim \mathcal D}[\mathcal R(x,y^T)]\\
    \text{where  }\quad&y^{1}\sim \pi_{\theta}\bigl(\cdot \mid x\bigr) \text{ and }|y^1|\le C\\
    &y^{t+1}\sim \pi_{\theta}\bigl(\cdot \mid [x,\, y^{t}]\bigr)\text{ and }|y^{t+1}|\le C, \quad \forall t = 1, \dots, T-1.
\end{align}
%leading to a unified \textbf{3D scaling framework}, as shown in Fig.~\ref{fig: test-time-scaling}. In 3D scaling, the whole process takes $T$ turns. Initially the prompt for the first turn is $p_0=x$. Each turn $t$ starts from prompt $p_t$, containing both the original problem and past experiences, and generates $B$ independent responses within a context length of $C$:
Along the turn dimension, the LLM is able to refine its reasoning by iteratively generating responses conditioned on the previous one. Each iteration allows the LLM to re-evaluate its prior trials and propose a new response, building up on past experiences. This method is called \textbf{Turn-Scaling (Reflection)} in this work. % Each iteration allows the model to reassess and improve upon its earlier conclusions, resulting in a lightweight form of self-improvement that often leads to more coherent and higher-quality responses.

\subsubsection{3D Scaling}
\paragraph{3D Scaling Framework.}  
% We formalize scaling along the turn dimension leading to a unified \textbf{3D scaling framework}, as shown in Fig.~\ref{fig: test-time-scaling}. In 3D scaling, the whole process takes $T$ turns. Initially the prompt for the first turn is $p_0=x$. Each turn $t$ starts from prompt $p_t$, containing both the original problem and past experiences, and generates $B$ independent responses within a context length of $C$:
We combine the previous three scaling methods, leading to a unified \textbf{3D scaling framework}, as shown in Fig.~\ref{fig: test-time-scaling}. In 3D scaling, the whole process takes $T$ turns. Initially, the prompt for the first turn is $p_0=x$. Each turn $t$ starts from prompt $p_t$, containing both the original problem and a \emph{context summary} of past experiences, and generates $B$ independent responses within a context length of $C$,
\[
\{y^t_1,y^t_2,\cdots,y^t_{B}\} \quad \text{where} \quad y_i^t \sim \pi_\theta(\cdot \mid p_t)\quad \text{and} \quad |y_i^t| \leq C \quad \forall i \in [1, B]
\]

Similar to batch scaling, an aggregation function $\mathcal{A}(\{y^t_1, \cdots, y^t_B\})$ is used to gather the $B$ responses and generate a context summary in turn $t$. 
The prompt for the next turn, $p_{t+1}$, is then composed by concatenating the input question $x$ and the context summary,
\[
p_{t+1} = [x, \mathcal{A}(\{y^t_1, \cdots, y^t_B\})]
\]

The final solution $y_{\text{final}}$ is then extracted from the aggregated result of the final turn, i.e. $A(\{y_1^T, \ldots, y_B^T\})$.

%\paragraph{3D Scaling with a Turn Dimension.} 
%We introduce an additional turn axis to batch scaling, leading to a unified 3D scaling framework, as shown in Fig.~\ref{fig: test-time-scaling}. In 3D scaling, the whole process takes $T$ turns. Each turn $t$ starts from prompt $p_t$, containing both the original problem and past experiences, and generates $B$ independent responses within a context length of $C$, $y^t_i\sim \pi_\theta(\cdot|p_t)$. Similar to the scoring function in batch scaling, an aggregation function $\mathcal {A}(\{y^t_1,\cdots, y^t_B\})$ is used to aggregate the $B$ responses and generate a \emph{context summary} that serves as a compact compression of the generation history in the first $t$ turns. The prompt for the next turn, $p_{t+1}$, is then composed by concatenating the input question $x$ and the context summary. The final answer is extracted from the aggregated result of the last turn $T$.

\paragraph{3D Scaling Implementations.} Different choices of aggregation functions and 3D configurations result in different implementations. We consider two major variants. 
\begin{itemize}
    % \item \textbf{Turn Scaling (Reflection):} This is a special case of 3D scaling with a single response per turn ($B=1$).  
    % The aggregation function is implemented via reflection by an LLM, which produces feedback on the model’s own output:
    % \begin{align}
    %     \mathcal{A}_{\text{Reflection}}(\{y_1^t\}) 
    %     = \pi_\theta(\cdot|[\text{Reflect over }y_1^t]).
    % \end{align}
    % The resulting reflection serves as additional context for the next turn.
    \item \textbf{3D Scaling (LLM Judge):}   In this setting, a batch of $B>1$ responses are generated in each turn. In each turn, a positive sample and a negative sample are selected as the context summary. To identify the best response, we adopt an LLM-based Best-of-N strategy. Specifically, we input all responses to the LLM and instruct it to return the optimal one. In addition to the optimal response, we also randomly sample a response among the rest as a negative example. This pair of positive and negative examples serves as the aggregated result for the latest turn.
    Formally,
    \begin{align}
        y_{\text{pos}}^t = \pi_\theta(\cdot|[\text{Select the best among }y_1^t, \ldots, y_B^t]),\quad y_{\text{neg}}^t \sim \mathrm{Unif}(\{y_1^t,\ldots,y_B^t\}\setminus\{y_{\text{pos}}^t\}))
    \end{align}
    \begin{align}
        \mathcal{A}_{\text{LLM-Judge}}(\{y_1^t,\ldots,y_B^t\}) 
        = (\,y_{\text{pos}}^t,\; y_{\text{neg}}^t\,).
    \end{align}
    The positive sample provides the best candidate from the previous step, effectively supplying the LLM with a stronger intermediate result to build upon. The negative sample provides a contrastive signal for the LLM to improve in the next turn. %enabling the LLM to clearly observe what should be improved relative to the positive example. Together, this contrastive signal helps guide the model more effectively toward improved response quality. 
    Finally, the final solution is extracted as the positive sample from the last turn, i.e., the selected $y_{pos}^T$.

    \item \textbf{3D Scaling (Human Judge):}  
    % Finally, we also consider an interesting instantiation of 3D scaling in a human-in-the-loop manner. In this approach, we use feedback from human experts as the aggregation function, where an expert evaluates the batch of responses and selects the best path forward. This approach is effective in cases when the LLM can not provide an accurate judgment to select the most salient response.
    In this setting, we examine a human-in-the-loop instantiation of 3D scaling. In each turn, human expert feedback serves as the aggregation function. The expert evaluates the batch of model responses and identifies both the most appropriate and the least appropriate responses as the output of the aggregation function. Formally, 
    
    \begin{align}
        y_{\text{pos}}^t, y_{\text{neg}}^t \leftarrow \text{The best and worst responses selected by the human expert}
    \end{align}
    \begin{align}
        \mathcal{A}_{\text{Human-Judge}}(\{y_1^t,\ldots,y_B^t\}) 
        = (\,y_{\text{pos}}^t,\; y_{\text{neg}}^t\,).
    \end{align}

    The final solution is taken as the positive sample chosen in the last turn, i.e., the human-selected $y_{pos}^T$. This approach is particularly effective when the language model is unable to reliably identify the most salient response.

\end{itemize}

We remark that it is also feasible to query the LLM to generate complex feedback for future turns, such as summarizations and reflections over the batch~\citep{shinn2023reflexion,huang2025gemini25procapable}. For simplicity, in this work, we select one or multiple responses as the aggregation result in each turn.  % Summarizations could synthesize key insights from different responses. Reflections can evaluate the effectiveness of different solutions and propose refinement directions for subsequent turns.

% \yw{\textbf{Choice of AGG function:} we can discuss the particular choice of AGG. simplest choice is selection function; other choices include reflection (e.g., reflexion) or summarization (e.g., cite some) or other}

% \yw{\textbf{Human as the AGG}: discuss the potential choice of human in the loop}

\input{03_pre}

\iffalse
\section{Method}\label{sec:method}

\input{04_method}
\fi

\section{Experiments}\label{sec:exp_new}

% \yw{
% \begin{enumerate}
% \item \textbf{experiment setup}: base model, setting, testbeds; better to follow the notation of $TTS(C,B,T)$ if needed.
% \item \textbf{multi-dimensiaonl test-time scaling}: run scaling experiment on the math problems. we need to have (1) scaling plots of each individual dimension (context; full-context + turn; full-context + batch); (2) scaling plots of combinations of turn and batch with full-context (we can (i) have a heatmap; or (ii)have a plot, x-axis is token, y-axis is accuracy, then for each batch size we have a curve responding to different choices of turns w.r.t. the particular batch size); (3) an optimal scaling plot with x-axis is the number of token, y axis is the accuracy (similar to the optimal reasoning efficiency paper, the merged curve); (4) suggestions and insights: we can say that each of these dimensions can be observed scaling effect but can be easily saturated; then raise the open question / research propose, we propose open questions to consider more dimensions or better aggregation / exploration methods to enable better in-context adaptation capacities.
% \item Performance on challenging tasks: show our best performances in tables, with IMO, IOI, CPHO, also report human-in-the-loop experiment
% \item Novelty discovery: show that the human-in-the-loop framework enables an interesting application is to discover novel embodied behavior, which can be never discovered sole by LLM. 
% \end{enumerate}
% }

\input{05_exp_new}
\section{Conclusion and Open Questions}\label{sec:conclu}

\input{06_conclu}


\bibliography{iclr2026_conference}
\bibliographystyle{iclr2026_conference}

\newpage
\appendix
% \section{Appendix}
\input{07_appendix}

\end{document}

%% file: math_commands.tex
\usepackage{amsmath,amsfonts,bm}

\def\eqref#1{equation~\ref{#1}}
\def\1{\bm{1}}

\DeclareMathAlphabet{\mathsfit}{\encodingdefault}{\sfdefault}{m}{sl}
\SetMathAlphabet{\mathsfit}{bold}{\encodingdefault}{\sfdefault}{bx}{n}

%% file: 00_abstract.tex
\begin{abstract}
Reasoning reinforcement learning (RL) has recently revealed a new scaling effect: test-time scaling. Thinking models such as R1 and o1 improve their reasoning accuracy at test time as the length of the reasoning \emph{context} increases. However, compared with training-time scaling, test-time scaling is fundamentally limited by the limited context length of base models, which remains orders of magnitude smaller than the amount of tokens 
consumed during training.
We revisit test-time enhancement techniques through the lens of scaling effect and introduce a unified framework of multi-dimensional test-time scaling to \emph{extend} the capacity of test-time reasoning. Beyond conventional context-length scaling, we consider two additional dimensions: \emph{batch scaling}, where accuracy improves with parallel sampling, and \emph{turn scaling}, where iterative self-refinement enhances reasoning quality.
Building on this perspective, we propose 3D test-time scaling, which integrates context, batch, and turn scaling. We show that: (1) each dimension demonstrates a test-time scaling effect, but with a bounded capacity; 
(2) combining all three dimensions substantially improves the reasoning performance of challenging testbeds, including IOI, IMO, and CPHO, and further benefits from human preference feedback; and (3) the human-in-the-loop framework naturally extends to a more open-ended domain, i.e., embodied learning, which enables the design of humanoid control behaviors.
%(2) combining all three dimensions substantially improves reasoning performance on challenging benchmarks such as IOI, IMO, and CPHO, and further benefits from human-in-the-loop selection; 
%and (3) the framework extends naturally to embodied learning, enabling the design of behaviors for robot control.
\end{abstract}

%% file: 01_intro.tex
Recent progress in reasoning reinforcement learning has introduced a new form of scaling effect by training thinking models such as R1~\citep{guo2025deepseek} and o1~\citep{openai2024o1}. Unlike conventional models that directly map input to output, a thinking model performs intermediate reasoning computation before producing its final answer. A striking phenomenon emerges during the reinforcement learning process: as the model is trained to reason over progressively longer contexts, its reasoning accuracy steadily improves~\citep{shi2025explainingcontextlengthscaling,aggarwal2025l}. At test time, this trend continues: extending the reasoning context length consistently leads to higher accuracy. This phenomenon is referred to as test-time scaling of reasoning models~\citep{muennighoff2025s1simpletesttimescaling}.

However, the potential of test-time scaling is fundamentally constrained by the context window size of current models. Even the most advanced commercial reasoning systems today support fewer than one million tokens of context—negligible compared with the scale of training-time compute, where tens of trillions of tokens are typically consumed during pretraining or post-training. This discrepancy naturally raises a question:
\begin{center}
\textit{How should we extend the capacity of test-time scaling?}
\end{center}

Notably, there have been many popular heuristics to enhance the reasoning model's performance at test time.
%Several strategies have been proposed to enhance the reasoning ability of models at test time. 
For example, majority voting improves accuracy by generating multiple candidate outputs in parallel and selecting the most frequent one~\citep{wang2023self}. Other approaches, such as Reflexion~\citep{shinn2023reflexion} and in-context learning~\citep{madaan2023self}, perform iterative self-refinement, where a model repeatedly revisits and improves its own solutions. Empirically, taking multiple refinement steps leads to a higher accuracy compared with directly outputting the solution. 
%increasing the number of refinement steps leads to higher accuracy. Yet, these methods have largely been studied as isolated heuristics, rather than through the lens of the scaling effect.

In this paper, we revisit these diverse techniques within a unified framework of \emph{multi-dimensional test-time scaling}. Specifically, we consider three dimensions: (1) Context scaling: reasoning accuracy improves with longer thinking context lengths; (2) Batch scaling: methods such as majority vote can be viewed as scaling along a batch dimension, where more parallel samples yield better aggregated answers; (3) Turn scaling: iterative refinement methods correspond to scaling along a turn dimension, where more refinement turns enhance accuracy. Each of these dimensions of scaling interacts with the context-length limits and capabilities of base LLMs, creating unique empirical trade-offs. % among the axes.

Building on this perspective, we propose \emph{3D test-time scaling}, which integrates all three dimensions: context, batch, and turn. We demonstrate that this unified view substantially extends the ceiling of test-time scaling compute and further enables a human-in-the-loop framework that applies to even open-ended domains.

%not only substantially extends the ceiling of test-time scaling effects but also unlocks new capabilities. In particular, we make three contributions:

\begin{itemize}
    \item We establish that each scaling dimension individually exhibits a test-time scaling effect: higher token consumption leads to higher accuracy. However, clear scaling limits can be observed for each dimension.%a scaling effect.
    \item We show that the unified 3D test-time scaling is capable of leveraging substantially more tokens for improved reasoning and achieving gold-level performances on challenging Olympiad competition problems, such as IMO and CPHO, and approaching competitive performance on IOI.
    %We show that the unified 3D test-time scaling is capable of leveraging substantially more tokens for improved reasoning and achieving gold-level performances on challenging Olympiad competition problems, such as IOI, IMO, and CPHO. T
    %achieves state-of-the-art reasoning performance on challenging benchmarks such as IOI, IMO, and CPHO. 
    The framework also extends to a human-in-the-loop setting, where a human operates along the batch dimension and selects the best candidate to further amplify final accuracy.
    \item Finally, we extend this human-in-the-loop framework to embodied learning, demonstrating that multi-dimensional test-time scaling enables models to interactively design open-ended behaviors in humanoid robot control.
\end{itemize}

%% file: 02_related.tex
% in context learning，test time scaling
% \yw{scaling law (including pretrain)}

% \yw{test-time compute (in-context method; prompting; tree search based; test-time training.)}

\paragraph{Scaling Effect.} %The cross-entropy loss in 
Large language model pretraining has been shown to scale predictably with key training resources, including model size, dataset size, and compute budget ~\citep{kaplan2020scalinglawsneurallanguage,rae2022scalinglanguagemodelsmethods,hoffmann2022trainingcomputeoptimallargelanguage}. With the emergence of thinking models such as DeepSeek-R1~\citep{guo2025deepseek} and OpenAI o1~\citep{openai2024o1}, researchers investigated training-time scaling beyond the number of training tokens. For example, \citet{shi2025explainingcontextlengthscaling} examines scaling behaviors with respect to context length. Scaling laws have also been studied at test-time. \citet{wu2025inferencescalinglawsempirical,snell2024scalingllmtesttimecompute} analyze how performance scales with respect to inference compute under different inference strategies such as majority voting and tree search, as well as tradeoffs between model size and test-time token budgets. In this paper, we focus on test-time scaling and propose a unified framework for characterizing the effects across three dimensions, context scaling, batch scaling, and turn scaling. In contrast, prior work on test-time scaling laws has typically examined only a subset of these aspects.

\iffalse
\paragraph{In-Context Learning.}
Large language models can learn new tasks at test time by conditioning on examples or other information in the input context, rather than through weight updates. 
This in-context learning ability was prominently revealed by GPT-3, where the model, given a prompt containing a few input-output demonstrations, could generalize to produce correct outputs for new inputs~\citep{brown2020language}. 
Following this discovery, numerous works have explored how and why LMs act as in-context learners. 
Some research treats in-context learning as a form of meta-learning occurring within the forward pass of the model. For instance, MetaICL~\citep{min2021metaicl} introduces a meta-training approach that explicitly teaches a model to rapidly adapt to new tasks from examples in its prompt, yielding stronger few-shot generalization. 
Other studies have investigated selecting or generating optimal exemplars to include in the prompt, as well as analyzing the internal mechanisms that enable prompt-based learning. 
% More recently, the scope of in-context learning has expanded beyond static examples to include self-generated context. In other words, a model’s own intermediate reasoning and feedback can serve as new context to improve its subsequent answers.
\fi

\paragraph{Test-Time Scaling.} Test-Time Scaling (TTS) refers to the class of algorithms for improving the model's performance through scaling inference-time compute. TTS methods can be broadly categorized into three approaches. 
\textit{Context scaling} methods improve performance through longer output sequences, exemplified by Chain-of-Thought prompting~\citep{wei2023chain}, which elicits step-by-step reasoning in large language models to improve performance on various benchmarks. 
Recent advances in reasoning models like o1~\citep{openai2024o1} and DeepSeek-R1~\citep{guo2025deepseek} further incentivize this ability, highlighting context scaling as an effective strategy for improving test-time performance.
% \textit{Batch scaling} approaches leverage parallel computation to explore multiple reasoning paths. Self-consistency decoding~\citep{wang2023self} generates multiple independent reasoning paths and aggregates the final answer through majority voting, aiming for diversity to increase the chance of a correct solution.
\textit{Batch scaling} approaches leverage parallel computation to explore multiple reasoning paths. Majority voting is a representative technique that leverages the power of parallel sampling~\citep{wang2023self} by generating multiple independent reasoning paths and selecting the majority final answer. 
Other work further incorporates test-time search~\citep{yao2023tree}, Monte-Carlo tree search~\citep{zhang2024rest,xie2024monte}, and parallel thinking~\citep{ning2023skeleton} to improve the performance.
\textit{Turn scaling} methods improve performance through iterative refinement, including Self-Refine~\citep{madaan2023self}, which enables models to iteratively improve outputs through self-feedback without additional training, and Reflexion~\citep{shinn2023reflexion}, which reinforces language agents through linguistic feedback and episodic memory to enhance future decision-making.
% Another line of work studies test-time training 

% \paragraph{Test-time Scaling.} Test-time Scaling (TTS) refers to the class of algorithms for improving the model's performance through scaling inference-time compute. TTS methods can be broadly categorized into parallel and sequential scaling. Parallel scaling, exemplified by Self-Consistency (SC), generates multiple independent reasoning paths and aggregates the final answer through majority voting, aiming for diversity to increase the chance of a correct solution. In contrast, sequential scaling methods guide later reasoning steps based on earlier results, improving depth but risking error propagation. Hybrid scaling aims to combine the benefits of both parallel and sequential scaling, with notable examples including beam search, Monte-Carlo Tree Search, and Tree-of-Thoughts. % A major limitation of both basic approaches is the computational overhead and the observation that naively scaling computation can sometimes degrade performance due to issues like "tunnel vision," where models get locked into an initial, potentially flawed reasoning path.

% \eugene{I think we should reference FunSearch since it does something similar to what we propose, just not entirely in the context}

%% file: 03_pre.tex
% test-time scaling

%% file: 05_exp_new.tex
%In this section, 
We begin with the experiment setup and then proceed with three evaluation stages. %, which includes the thinking model and two testbeds: problem-solving tasks (IMO, CPHO, IOI) and innovative embodied tasks. Our experiments proceed in three stages: 
First, we examine the performance of different test-time compute configurations over three dimensions on IMO problems to illustrate the test-time scaling phenomena. %examine the accuracy on the IMO problems by of different configurations of the three dimensions of the IMO problems to show the test-time scaling effect. %examine the scaling effect on reasoning performances on IMO problems by varying the three scaling dimensions to show the. %s s of scaling on reasoning performance using IMO problems; 
Next, we explore how the unified 3D scaling pushes the reasoning capacity on a collection of challenging Olympiad problems. %limits on problem-solving tasks; and 
Finally, we extend the framework to a more open-ended setting, embodied learning. With human feedback in the loop, 3D scaling produces robotic control behaviors that are more aligned with human preferences. %assess the performance of 3D Scaling with human-in-the-loop feedback on the innovative embodied tasks.

\subsection{Experiment Setup}

\paragraph{Base Reasoning Model:} We conduct all experiments using Gemini 2.5 Pro~\citep{comanici2025gemini25pushingfrontier} as the backbone model, chosen for its strong reasoning and coding capabilities in complex problem-solving tasks. To ensure reproducibility, the temperature is fixed at $0.1$, yielding highly deterministic outputs across trials. For each domain, we further design tailored system prompts for solution generation and feedback learning; full prompt details are provided in Appendix~\ref{app:icpl-system-prompt}.

\textbf{Testbeds:} We explore the scaling effect on two types of testbeds:  

\begin{itemize}
    \item \textbf{Reasoning Problem-Solving Tasks:} This testbed focuses on rigorous reasoning and algorithmic problem-solving. (1) \textit{Math and Physics Olympics:} We adopt problems from the IMO~\citep{IMO2025} and CPHO~\citep{CPHO2022} to evaluate the LLM’s reasoning capabilities. 
    %A solution is considered correct only if both the final answer and the entire reasoning process are mathematically valid, with their correctness verified by experts. Additional details on problem selection are provided in Appendix~\ref{app:cpho_vs_ipho}. 
    (2) \textit{Coding:} IOI 2025 problems~\citep{IOI2025} are used to assess programming ability under 3D Scaling. Unlike human contestants who receive submission feedback, the LLM must directly solve tasks without intermediate guidance. 
    
    \item \textbf{Innovative Tasks:} This testbed targets embodied AI and emergent behaviors. %, where we additionally consider human-in-the-loop settings. 
    We use several robotics reinforcement learning tasks from GPU-based IsaacGym~\citep{makoviychuk2021isaacgymhighperformance} that cover diverse environments. We also introduce a new task, \textit{HumanoidJump}, which aims to make a humanoid jump in a human-like manner. Designing a reward for this task is an open challenge because human-like jumping lacks easily quantifiable criteria.
    
\end{itemize}

\paragraph{Evaluation: } For IMO and CPHO problems, every LLM-generated solution is rigorously verified by \emph{human experts} following the scoring guidance. A solution is considered correct only if both the final answer and the entire reasoning process are mathematically valid. For IOI problems, the score is measured over the official IOI test cases. For innovative tasks, we recruit human volunteers to vote for their preferred behaviors. 

\paragraph{Batch Scaling Aggregation Methods:}
\begin{itemize}
    \item In the Batch Scaling setting under the single-turn ($T=1$) configuration for IMO and CPHO tasks, we apply a majority-vote procedure followed by a best-of-all selection. The motivation is that, when the batch size becomes large (e.g., 30 responses), the LLM does not have access to the ground truth and may struggle to reliably identify the single best solution. Majority voting therefore stabilizes the aggregation by filtering out noisy or inconsistent candidates before applying the final selection.

    \item In the Batch Scaling setting on IOI tasks, since the code generated by LLMs can vary significantly and is difficult to vote on, we directly employ the best-of-N strategy. We provide baseline results for both Scoring-based Best-of-N and LLM-based Best-of-N strategies to examine the effectiveness of batch scaling with and without an external ground-truth verifier.
    
    \item In the Batch Scaling setting for the multi-turn task (3D scaling), we rely on LLM-based selection to identify the best response at each turn. Because each round generates only five candidate solutions, LLM choosing remains both feasible and efficient. %Moreover, there is no well-defined or reliable scoring function for intermediate proof trajectories, making vote-based aggregation more appropriate for this task.

\end{itemize}

\paragraph{Human-in-the-loop Feedback:} In the setting of 3D scaling with $B>1$, in addition to using an LLM judge, we can also introduce a human judge to select the best solution among all parallel candidates in each refinement turn according to the task objective. % and the most deficient solutions among candidate responses 
%based on whether they satisfy the task objectives and whether they can be further improved. 
Details about evaluators are presented in Appendix~\ref{app:human-evaluation}.

\subsection{Performance Analysis of Test-Time Scaling on Each dimension}

In this subsection, we study test-time scaling on the IMO benchmark. We select three moderately difficult problems (1, 3, and 5), excluding those that are too easy or too hard. Each is tested over five trials, and we report the average accuracy over 3 problems. To fully exploit the backbone LLM, all experiments except \textbf{Context Scaling} fix the context length at 32k.

\subsubsection{Single-Dimension Scaling Analysis}

We investigate the test accuracy by scaling along each of the three dimensions. %, each instantiated with a standard method:
For \textbf{Context Scaling}, we vary context length $C$ from 1k to 32k. For \textbf{Batch Scaling (Vote)}, we take the 32k context length with $B$ parallel rollouts ranging from 1 to 30.
For \textbf{Turn Scaling (Reflection)}, we adopt full context and $B=1$ while allowing the model to take 1 to 10 refinement turns.

Fig.~\ref{fig: individual-scaling} reports the average accuracy as a function of total thinking budget across the three individual scaling dimensions. Performance improves at small scales but quickly plateaus, with little or no gain from further scaling. In particular, extending the context length beyond a moderate range yields little improvement, and increasing the number of turns offers diminishing returns. Notably, performance under batch scaling even \textbf{degrades} at large $B$ (e.g., $B=30$), suggesting that naive aggregation may not always help. 

We hypothesize that this degradation arises from systematic biases in the model’s reasoning process: when the model consistently favors a specific incorrect derivation pattern, majority voting may amplify the bias instead of correcting it. %This phenomenon is further examined in Sec.~\ref{sec: major-vote-analysis}.%\yw{[can we put a citation here, say similar phenomena are also reported in xxx, yyy????]}

\begin{figure}
	\centering
	% \vspace{-12mm}
	% \hspace{-2mm}
 \includegraphics[width=0.8\textwidth]{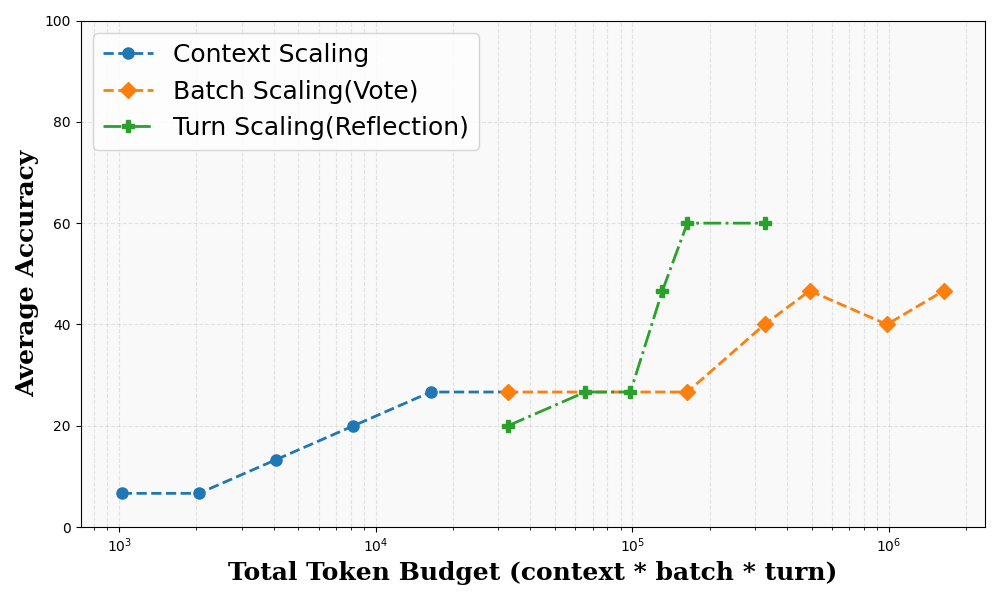}
	\centering 
	\vspace{-2mm}
    \caption{The average accuracy over the IMO2025 dataset as a function of the total thinking budget for individual scaling on three dimensions: context, batch and turn. All three scaling methods achieve substantial improvements at small scales but saturate as the scale becomes larger.}
	% \vspace{-4mm}
    \label{fig: individual-scaling}
\end{figure}

\subsubsection{Additional Analysis: Accuracy Dropping in Majority Vote}

\label{sec: major-vote-analysis}

\begin{figure}[t]
    \centering
    % \vspace{-3mm}
    \begin{subfigure}[t]{0.48\textwidth}
        \centering
        \includegraphics[width=\linewidth]{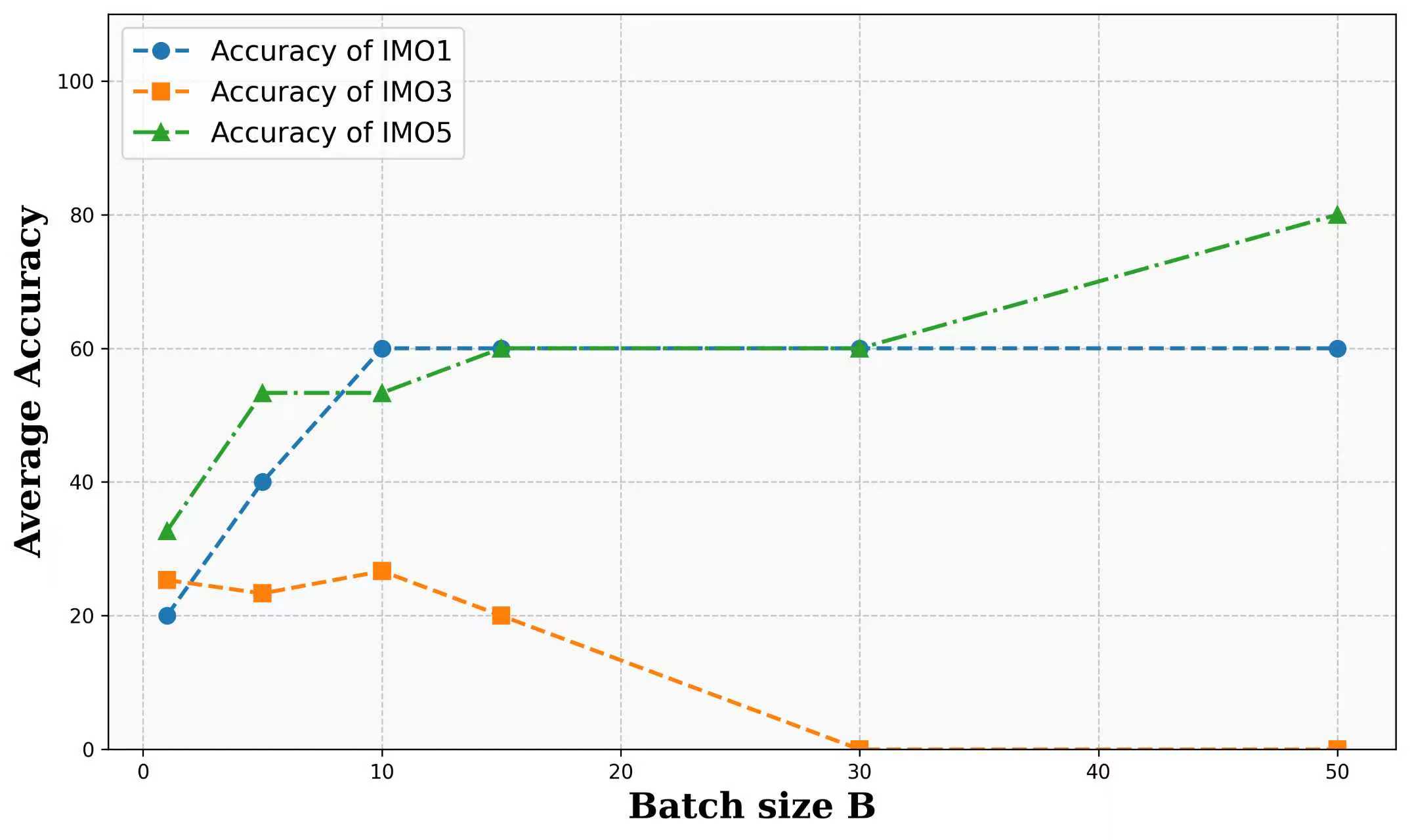}
        \caption{Average accuracy}
        \label{fig: imo-problems}
    \end{subfigure}
    \hfill
    \begin{subfigure}[t]{0.48\textwidth}
        \centering
        \includegraphics[width=\linewidth]{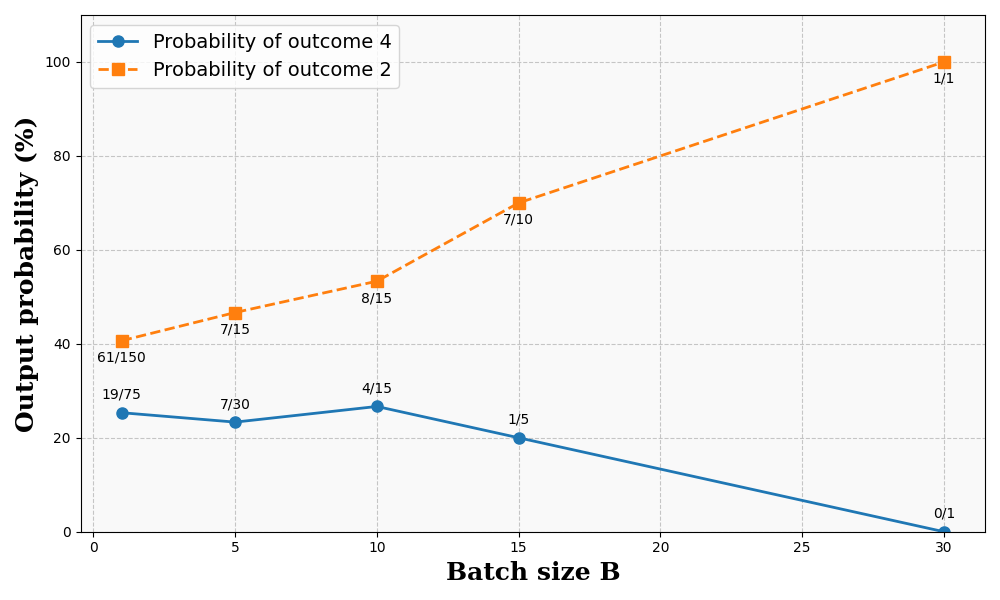}
        \caption{The failure mode of majority vote in IMO3}
        \label{fig: anomaly}
    \end{subfigure}
    % \vspace{-2mm}
    \caption{
        Batch scaling analysis on IMO problems.
        (a) shows the average accuracy of batch scaling (majority vote) on each IMO problem with different batch sizes.
        (b) illustrates the failure mode of majority vote observed in IMO3. The model produces both the correct answer ``4'' and the incorrect answer ``2''. As batch size $B$ increases, the probability of selecting the distractor ``2'' grows due to model bias.
    }
    \vspace{-4mm}
    \label{fig: batch-scaling-imo}
\end{figure}

A counter-intuitive phenomenon was observed in the batch scaling experiments on the IMO task (Fig.~\ref{fig: individual-scaling}): the test accuracy of \textbf{Batch Scaling (Vote)} decreased as the batch size $B$ increased beyond 15. 
This contradicts the usual ensemble-learning intuition that aggregating more samples should reduce variance and improve accuracy.

Upon examining the experimental results, we observed that the performance on both IMO1 and IMO5 improved with increasing batch size, with the growth curve plateauing when the batch size is sufficiently large. However, IMO3 exhibited a counterintuitive trend where accuracy consistently decreased as the batch size increased, as shown in Fig.~\ref{fig: batch-scaling-imo}(a). % need add a new figure
We analyze the results on IMO3 task and explore a bias amplification effect on response chosen by majority vote, as shown in Fig.~\ref{fig: batch-scaling-imo}(b). 
The model shows a consistent preference for the incorrect answer ``2'' over the correct answer ``4'', i.e.,
\[
\Pr(\text{2 (incorrect)} \mid \text{IMO3}) > \Pr(\text{4 (correct)} \mid \text{IMO3}).
\]
As the batch size $B$ increases, the empirical vote proportion for ``2'' dominates, and the aggregated prediction $\hat{a}_B$ increasingly favors the wrong answer, leading to a drop in test accuracy.

To understand this behavior, we formalize the condition under which majority voting can amplify model bias and thereby hurt accuracy.

% ---------------------- THEOREM ----------------------
\begin{theorem}[Systematic Bias Amplification under Majority Voting]
\label{thm:majority-bias}
Given an LLM policy $\pi_\theta$, an input question $x\in\mathcal X$, and a unique ground-truth answer $a^*$. If some incorrect answer $\tilde{a}$ has strictly higher probability of being produced by the LLM than $a^*$, then the accuracy of
Majority Voting approaches zero as the batch size grows.  
Formally,
\[
\lim_{B\to\infty}
\Pr_{y_1,\dots,,y_B\sim \pi_{\theta}(\cdot\mid x)}\!\left[\mathcal{A}_{\mathrm{Vote}}(\{y_1,\ldots,y_B\}) = a^*\right] = 0.
\]
\end{theorem}

\begin{proof}
See Appendix ~\ref{app: proof1}.
\end{proof}

This finding reveals that batch scaling with majority voting is not inherently reliable and can even be detrimental when the underlying model exhibits systematic response biases. Instead of mitigating randomness, larger batch sizes may magnify these biases and push predictions further away from the correct answer. This highlights the need for alternative aggregation strategies or for scaling along other dimensions to ensure that increasing the inference scale actually improves performance rather than undermining it.

\begin{figure}
	\centering
	\vspace{-4mm}
	% \hspace{-2mm}
 \includegraphics[width=0.8\textwidth]{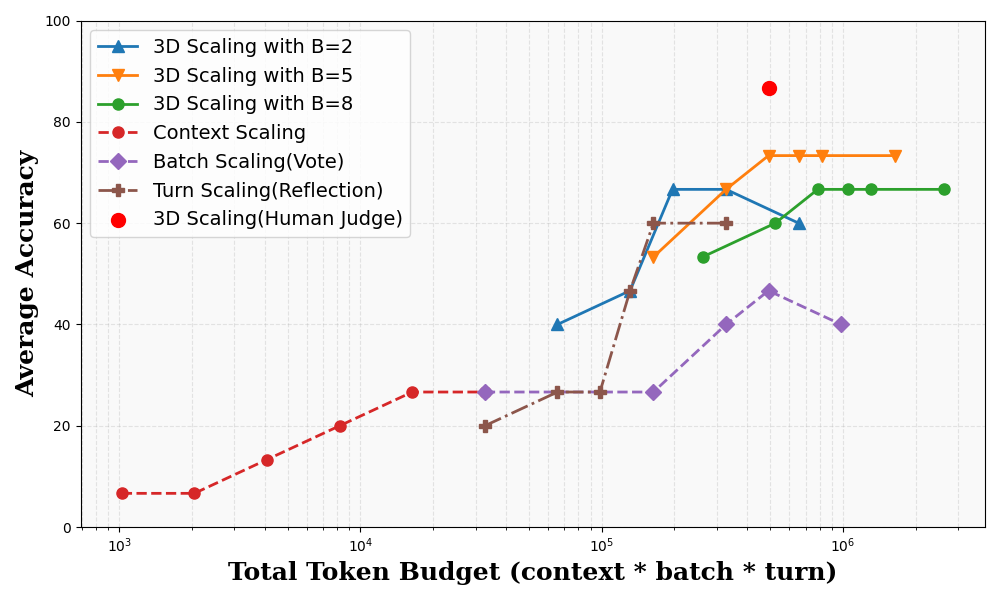}
	\centering 
	\vspace{-4mm}
    \caption{The average accuracy over the IMO2025 dataset as a function of the total thinking budget for individual scaling and 3D Scaling with different batch sizes. 3D Scaling achieves performance beyond the limits of individual scaling, reaching 73.3\%. The red marker denotes 3D Scaling with a human judge, which attains 86.7\% accuracy, highlighting the effectiveness of human feedback.
    % \eugene{move legend to upper left hand side.}
    }
	\vspace{-4mm}
    \label{fig: all-comparison}
\end{figure}

\subsubsection{3D Scaling Analysis}

We conducted 3D Scaling experiments that combine batch scaling and turn-based scaling, using a simple preference aggregation function provided by the LLM Judge. The three solid curves in Fig.~\ref{fig: all-comparison} about 3D Scaling with various batch sizes show how model accuracy varies with two parameters: the \textit{batch size} $B$ and the number of \textit{turns} $T$. The plotted results correspond to the average accuracy over the IMO2025 dataset as a function of the total thinking budget. Notably, 3D Scaling exceeds the performance limits of individual scaling, reaching an accuracy of 73.3\%.

The results largely align with those from single-dimension scaling. Increasing the number of turns $T$ initially improves performance by stabilizing predictions. However, further increases lead to saturation and may even reduce accuracy, likely because an incorrect judgment in one turn can propagate through subsequent refinements.

Increasing batch size $B$ from 1 to 5 substantially improves overall performance, while further increasing it to 8 results in a performance drop. In particular, $B=5$ yields the strongest performance among all evaluated settings, outperforming both $B=2$ and $B=8$. Moreover, under $B=5$, the achieved accuracy exceeds the baseline by more than a factor of three, underscoring the effectiveness of moderate batch sizes. We conjecture that the decline observed at $B=8$ may be due to the LLM failing to correctly identify the best and worst solutions when the number of candidate solutions grows, highlighting an open problem of how to perform this selection optimally.
% Increasing the batch size from 1 to 2 substantially improves performance, while further increases to 5 yield minimal gains, and performance drops at $B=8$, reflecting the accumulation of incorrect judgments.

Fig.~\ref{fig: all-comparison} compares 3D Scaling with baseline methods across different parameter settings. The results show that 3D Scaling effectively leverages the reasoning capabilities of the LLM, achieving a maximum average accuracy of 73.3\%. We also report the outcome of applying 3D Scaling with human judgment under the setting $C=32768, B=5, T=3$, where the score rises to 86.7\%. The red marker in Fig.~\ref{fig: all-comparison} highlights this result, demonstrating the substantial benefits of incorporating human feedback.

\textbf{Insights:} From the results, we summarize two key insights:

\begin{itemize}
    \item \textbf{Scaling Saturation.} 
    Performance improves along all three individual scaling dimensions—context, batch, and turn—but only up to a limited extent. 
    Context scaling quickly reaches a plateau due to bounded information utilization, turn scaling (reflection) yields diminishing returns after several iterations, and batch scaling even causes a performance drop when the batch becomes too large.
    %(Since the major vote result may not match the ground truth.)

    \item \textbf{Beyond Single-Dimension Boundaries.} 
    Combining multiple scaling dimensions is observed to surpass the performance ceiling of any single dimension. 
    This suggests that different forms of scaling may complement one another rather than overlap in effect. 
    Whether there exist additional, yet unexplored, scaling dimensions that could further extend this frontier remains an open and important research question for developing more capable reasoning systems.
\end{itemize}

\subsection{Evaluating 3D Scaling on Benchmark Tasks}

In this subsection, we present 3D scaling experiments with selection feedback from both LLM and human judges on three challenging benchmarks, using a setting of $B=5, T=3$. For statistical reliability, we conducted 5 independent trials for each comparative methods. To ensure fairness, batch scaling was configured to generate 15 solutions per trial, thereby matching the total token budget of the 3D scaling setup. We present the average final scores across different trials on all benchmarks using radar charts in Fig.~\ref{fig:radar_all}.
% \eugene{Leaving myself a note that t his sentence needs to be written carefully}

% caption 仍需修改，IsaacGym 的雷达图略小，稍后修改
\begin{figure}[htbp]
    \centering
    % 第一行：IMO + CPHO
    \begin{subfigure}[b]{0.48\textwidth}
        \centering
        \includegraphics[width=\textwidth]{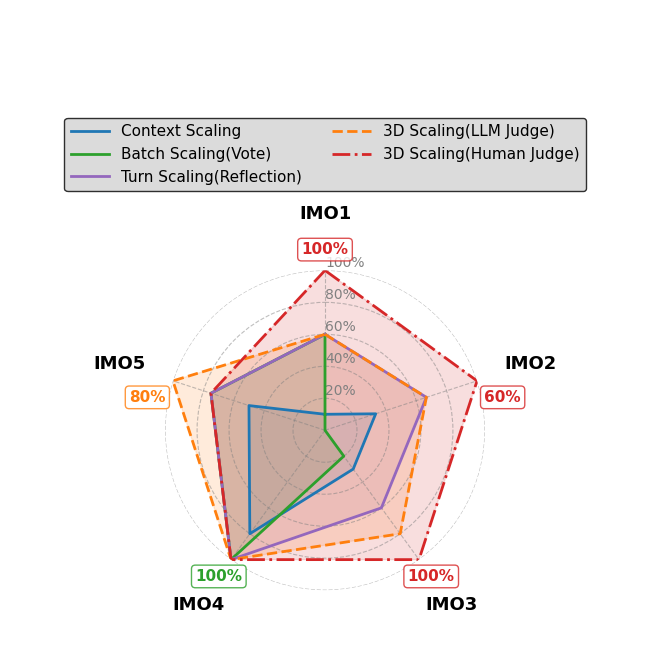}
        \caption{Average accuracy on the IMO 2025 benchmark}
        \label{fig:imo_radar}
    \end{subfigure}
    \hspace{0.02\textwidth}
    \begin{subfigure}[b]{0.48\textwidth}
        \centering
        \includegraphics[width=\textwidth]{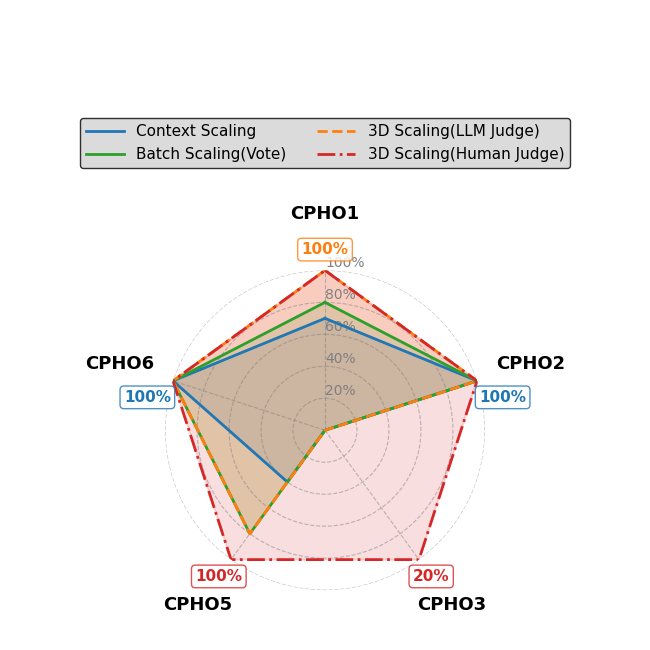}
        \caption{Average accuracy on the CPHO 2022 benchmark}
        \label{fig:cpho_radar}
    \end{subfigure}

    \vspace{0.02\textwidth}

    % 第二行：IOI + IsaacGym
    \begin{subfigure}[b]{0.48\textwidth}
        \centering
        \includegraphics[width=\textwidth]{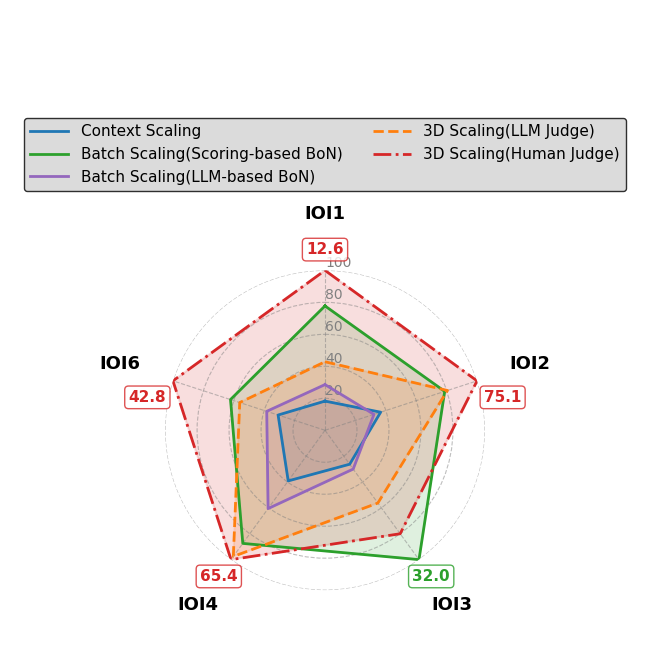}
        \caption{Average score on the IOI 2025 benchmark}
        \label{fig:ioi_radar}
    \end{subfigure}
    \hspace{0.02\textwidth}
    \begin{subfigure}[b]{0.48\textwidth}
        \centering
        \includegraphics[width=\textwidth]{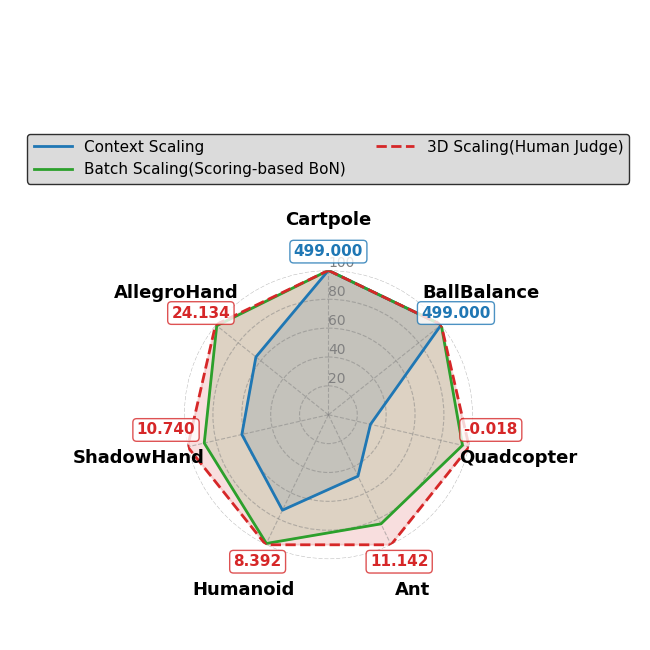}
        \caption{Average score on the IsaacGym benchmark}
        \label{fig:isaacgym_radar}
    \end{subfigure}

    \caption{
    Comprehensive comparison of different test-time scaling methods across four domains: Math Olympics\;(IMO2025), Physics Olympics\;(CPHO2022), Coding\;(IOI2025), and Embodied\;(IsaacGym). Each dimension in the radar charts represents a single task or problem and is normalized by the best-performing method on that specific dimension.
    3D Scaling with a human judge consistently outperforms baseline methods including context scaling, turn scaling, and batch scaling, across different benchmarks. 3D Scaling with LLM judge also achieves competitive results on the IMO 2025 and CPHO 2022 benchmarks, but performs worse than 3D Scaling with a human judge on the challenging programming task. %, overcoming saturation limits and generating complete reasoning through iterative refinement. 3D Scaling with LLM feedback also achieves competitive results.
    % Batch Scaling achieves competitive results by directly selecting high-scoring solutions but remains below 3D Scaling when using the same number of thinking tokens. 
    Results for IMO6 and CPHO4 are excluded due to zero accuracy across all methods.
    (Since IMO2 is a fully proof problem, it is impossible to do Gemini vote for a proof process. So we have not done the Batch Scaling(Vote) experiment for IMO2. )
    }
    \label{fig:radar_all}
\end{figure}

\subsubsection{Math Olympics}

The performance of different test-time scaling methods on all six problems in IMO 2025 (\cite{IMO2025}) is summarized in Table~\ref{tab:imo}. The experimental results reveal several key observations. The single-response \textbf{Context Scaling} approach achieved moderate performance. Analysis of the responses indicates that while the model can produce reasonable answers over multiple runs, it often generates incomplete or partially valid reasoning. \textbf{Batch Scaling (Vote)} and \textbf{Turn Scaling (Reflection)} improve accuracy over the context scaling baseline by scaling along individual dimensions. However, both methods reach saturation when the scale increases to 15, and the model’s ability to produce fully complete reasoning remains limited.
The fully automated iterative refinement approach, \textbf{3D Scaling (LLM Judge)}, demonstrates competitive performance, achieving higher accuracy than the baseline scaling methods. This suggests that scaling across multiple dimensions can overcome the limitations of single-dimension scaling. 
Furthermore, applying \textbf{3D Scaling (Human Judge)} leads to substantial improvements over all baselines, achieving the best overall performance. Incorporating human feedback addresses the LLM’s tendency to produce incomplete reasoning, enabling it to generate solutions with fully correct reasoning through iterative refinement.

\input{tables/imo.tex}

\subsubsection{Physics Olympics}
The performance of different test-time scaling methods on all six problems in CPHO 2022 is summarized in Table~\ref{tab:cpho}. The results on physics competition problems demonstrate a trend consistent with that observed in mathematical competitions: \textbf{3D Scaling (Human Judge)} achieves the highest accuracy, followed by \textbf{3D Scaling (LLM Judge)}, \textbf{Batch Scaling (Vote)}, and finally the single-response \textbf{Context Scaling}. Extra analysis is provided in Appendix~\ref{app: pho}.

\input{tables/cpho.tex}

\subsubsection{Coding}

The test results of different test-time scaling methods on IOI 2025 are presented in Table~\ref{tab:ioi}. Among the six problems in IOI 2025, the fifth problem is a communication task; since the backbone model cannot access the submission API, its performance on this task is unsatisfactory, and we therefore exclude it from evaluation. For reference, the bronze medal cutoff at IOI 2025 is 252 points, and a score of 221.53 corresponds to roughly the top 60\%. This lower performance arises because, without the ability to test code correctness, LLMs face substantial difficulty in producing higher-scoring solutions.

\input{tables/ioi.tex}

The results indicate that 3D Scaling can substantially enhance coding performance through human feedback. On IOI problems, LLMs often struggle to generate fully correct solutions in a zero-shot setting. Consequently, \textbf{Context Scaling} typically solves only a subset of tasks and sometimes contains errors in complexity analysis. Because the codes generated by the LLM vary significantly across IOI problems, we adopt both \textbf{Batch Scaling (LLM-based BoN)} and $\textbf{Batch Scaling (Scoring-based BoN)}$ that chooses the best solution from responses to demonstrate the upper bound of Batch Scaling.

When feedback from a human is incorporated, 3D Scaling demonstrates the best performance. Even when all early solutions are incorrect, the \textbf{3D Scaling (Human Judge)} can identify issues in the code and iteratively refine them, enabling it to solve more tasks and achieve higher scores. Across nearly all experiments, this approach produces final scores that surpass the best first-round solutions, achieving an average improvement of approximately 19.9\% over the \textbf{Batch Scaling (Scoring-based BoN)} baseline.

Because the test API is inaccessible and the problems are relatively difficult, the LLM’s ability to select the best solution is less reliable, leading to a large performance gap between \textbf{Batch Scaling (LLM-based BoN)} and \textbf{Batch Scaling (Scoring-based BoN)}. Nevertheless, \textbf{3D Scaling (LLM Judge)} still consistently outperforms the \textbf{Context Scaling} approach and achieves comparable overall performance. Although it is less precise than \textbf{3D Scaling (Human Judge)} on challenging tasks, these results highlight the feasibility and effectiveness of auto-feedback mechanisms for improving code generation, even in the absence of human feedback.

\subsection{Experiments on Innovative Tasks}
\label{sec:embodied-details}

In this section, we evaluate the effects of human feedback on several robotics reinforcement learning tasks using the GPU-based IsaacGym framework (\cite{makoviychuk2021isaacgymhighperformance}), including \textit{Cartpole, BallBalance, Quadcopter, Ant, Humanoid, ShadowHand, and AllegroHand}, along with a challenging and innovative new task, \textit{HumanoidJump}, defined as “making a humanoid jump like a real human”, which is an open-ended challenge without gold-standard answers.

We employed GPT-4o as the backbone model. The model was prompted to generate task-specific reward functions, which were then used to train agents in the simulator.
In these tasks, we employed settings of $B=6$ and $T=5$ for \textbf{3D Scaling (Human Judge)}. In each iteration, the evaluators selected the best and worst reward functions based on behavior videos of the agents trained with these reward functions. Details about this process and evaluators are provided in Appendix~\ref{sec: innovative-human}. We also report the results with the \textbf{Context Scaling} and \textbf{Batch Scaling (Scoring-based BoN)}.

In each turn, in addition to providing preference feedback, we also generate automatic feedback with LLM, which is combined with human preferences as the feedback prompt for the next round to assist the LLM in refinement. The automatic feedback consists of the following three components:
\begin{itemize}
\item \textbf{Evaluation of reward functions}: The component values that make up the good and bad reward functions are obtained from the environment during training and provided to the LLM. This helps the LLM assess the usefulness of different parts of the reward function by comparing the two.
\item \textbf{Differences between historical reward functions}: We employed GPT-4o to analyze the differences between the historically best reward functions from each iteration. These differences were then provided to the generator LLM to assist in refining the reward function.
\item \textbf{Reward trace of historical reward functions}: The reward trace, consisting of the values of the good reward functions during training from all prior iterations, is provided to the LLM. This reward trace enables the LLM to evaluate how well the agent is actually able to optimize those reward components.
\end{itemize}

\subsubsection{Task Metric}

For evaluation, we used the reward function in a PPO~\citep{schulman2017proximalpolicyoptimizationalgorithms} training loop following the original setting in IsaacGym, and reported the average task score, measured by the expert-written task metrics across multiple experiments, as the ground truth rewards for each method. We also directly use this ground truth rewards for response selection in baseline \textbf{Batch Scaling(Scoring-based BoN)}.
The details of the task metrics are provided in Appendix\ref{app: isaac-metrics}.
For the \textit{HumanoidJump} task, since designing a reward metric is challenging, we adopt human votes for quantitative evaluation instead, which is detailed in Sec. \ref{sec: humanoid-jump}.

\subsubsection{IsaacGym Tasks Results}
% \vspace{-2mm}

For each environment, we conducted five runs per method and reported the average ground-truth rewards in Table ~\ref{tab: isaac}, while ensuring that \textbf{Batch Scaling (Scoring-based BoN)} and \textbf{3D Scaling (Human Judge)} used the same total token budget. As observed, \textbf{3D Scaling (Human Judge)} significantly outperforms \textbf{Batch Scaling (Scoring-based BoN)} in 3 out of 5 challenging tasks, achieving an average improvement of 18.4\%. In addition, we conducted another set of experiments with a proxy judge and analyzed performance improvements across turns, as detailed in Appendix~\ref{app: isaac}.

\input{tables/Isaacgym.tex}

% While it is possible that the LLMs could generate an optimal reward function in a zero-shot manner, the primary focus of our analysis is not solely on absolute performance values. Rather, we emphasize whether 3D scaling is capable of enhancing performance through the iterative incorporation of preferences. We calculated the average RTS improvement compared to the first iteration for the two tasks with the largest improvements compared with ``OpenLoop'', \textit{Ant}, and \textit{ShadowHand}. As shown in Fig. \ref{fig: impro}, RTS demonstrates improved performance after multiple iterations (e.g., 5 vs. 1), highlighting its effectiveness in refining reward functions. 

% \begin{figure}
% 	\centering
% 	% \vspace{-12mm}
% 	% \hspace{-2mm}
%  \includegraphics[width=0.5\textwidth]{figures/impro.png}
% 	\centering 
% 	\vspace{-4mm}
%     \caption{Average improvement of the Reward Task Score (RTS) compared with the first iteration in 3D scaling-Proxy Judge for the Ant and ShadowHand tasks, demonstrating the method's effectiveness in refining reward functions.}
% 	\vspace{-4mm}
%     \label{fig: impro}
% \end{figure}

\subsubsection{HumanoidJump Task Results}
\label{sec: humanoid-jump}

\begin{figure}[ht]
	\centering
	% \vspace{-6mm}
	% \hspace{-2mm}
		\includegraphics[width=0.6\linewidth]{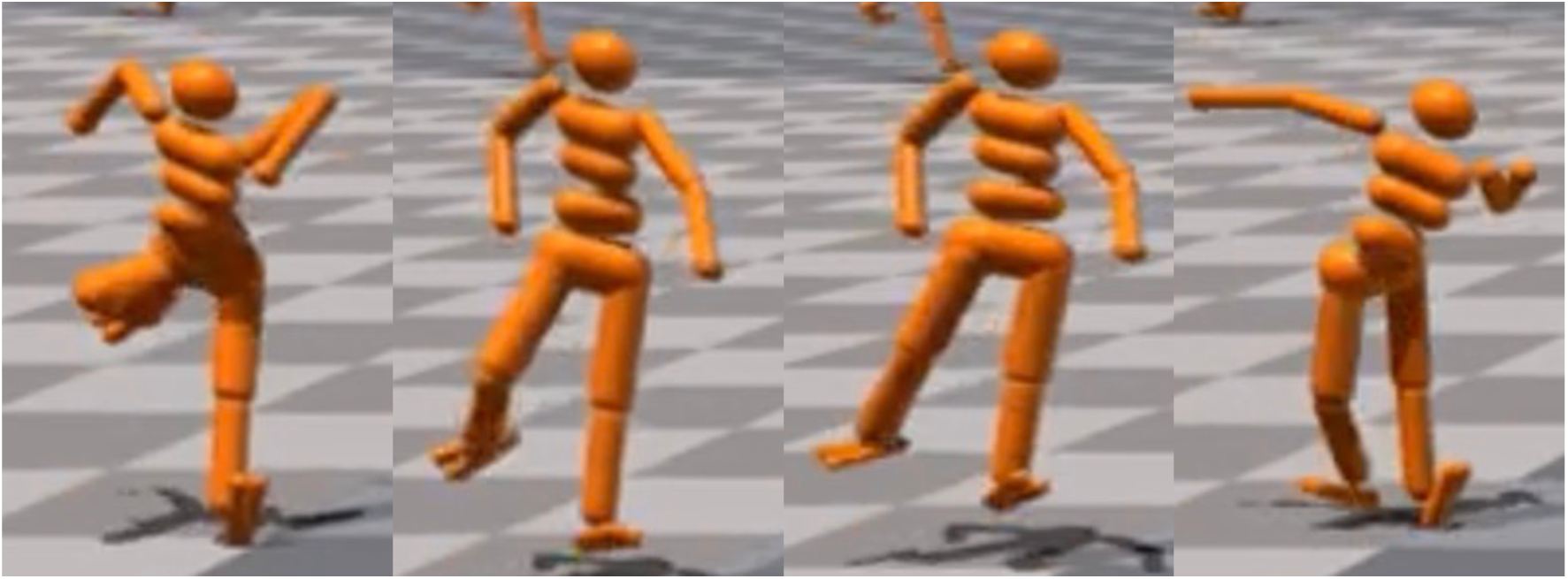}
	\centering 
	% \vspace{-6mm}
    \caption{A common behavior.}
    \label{fig: iter}
\end{figure}

\begin{figure}[ht]
    \centering
    % \vspace{-2mm}
    \includegraphics[width=\linewidth]{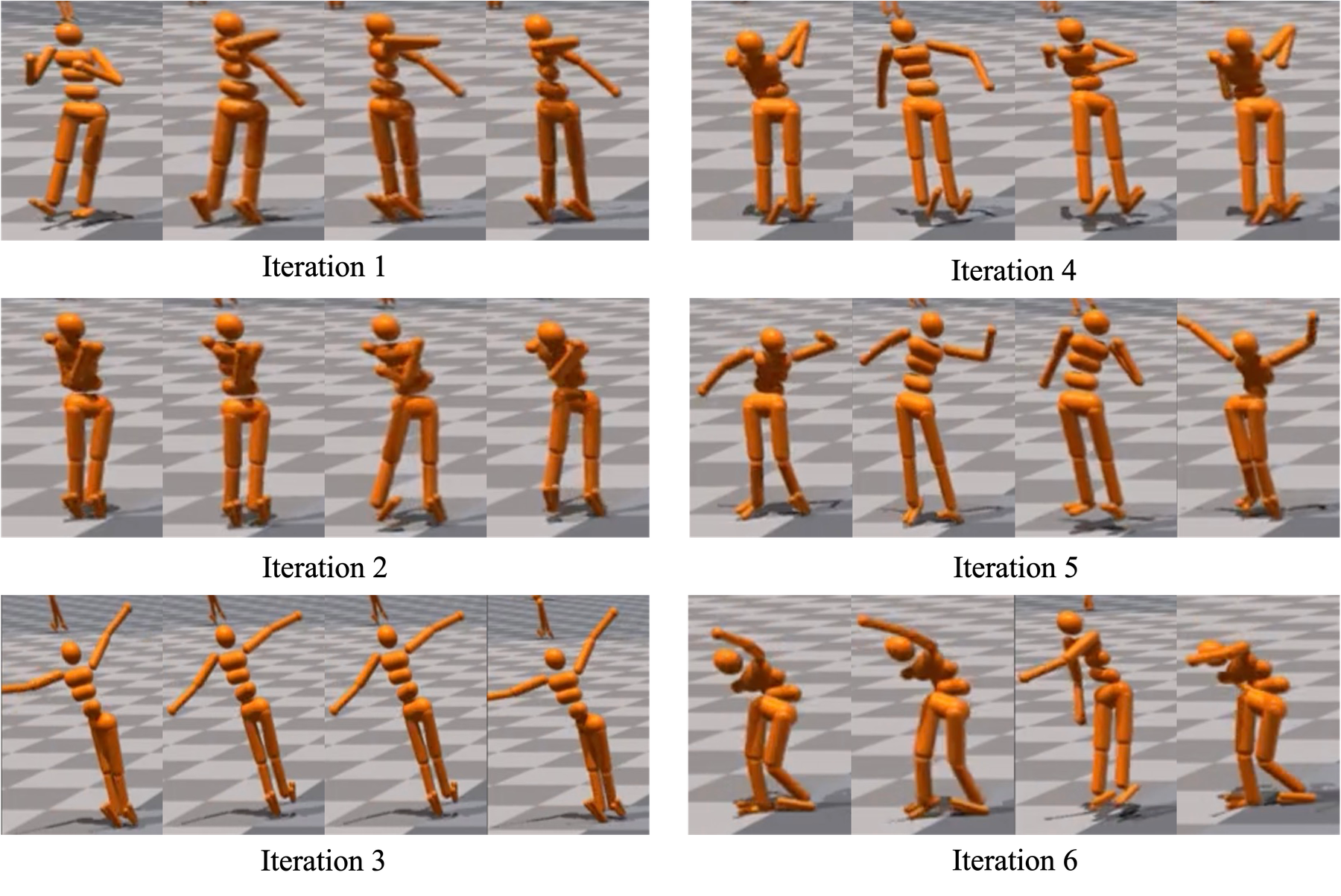}
    \vspace{-5mm}
    \caption{The humanoid learns a human-like jump by bending legs and lowering the upper body to shift the center of mass in a trial of human-in-the-loop 3D Scaling experiments. Note that both legs are used to jump, and the agent bends at the hips.}
    \label{fig: jump}
\end{figure}

Without human feedback, the most common behavior observed in this task, as illustrated in Fig.~\ref{fig: iter} is what we refer to as the ``leg-lift jump.'' This behavior involves initially lifting one leg to raise the center of mass, followed by the opposite leg pushing off the ground to achieve lift. The previously lifted leg is then lowered to extend airtime. Various adjustments of the center of mass with the lifted leg were also noted. This behavior meets the minimal metric of a jump: achieving a certain distance off the ground.
If feedback were provided based solely on this minimal metric, the ``leg-lift jump'' would likely be selected as a candidate reward function. However, 
% our experiments demonstrate that 
such candidates show limited improvement in subsequent iterations, failing to evolve into more human-like jumping behaviors.

Conversely, when real human preferences were used to guide the task, the results were notably different. The volunteer judged the overall quality of the humanoid's jump behavior instead of just the metric of leaving the ground. Fig.~\ref{fig: jump} illustrates that the volunteer successfully guided the humanoid towards a more human-like jump by selecting behaviors that, while initially not optimal, displayed promising movement patterns. 
 
% The reward functions are shown in Appendix \ref{app: reward}.
In the first iteration, ``leg-lift jump'' was not selected despite the humanoid jumping off the ground. Instead, a video where the humanoid appears to attempt a jump using both legs, without leaving the ground, was chosen. By the fifth and sixth iterations, the humanoid demonstrated more sophisticated behaviors, such as bending both legs and lowering the upper body to shift the center of mass, behaviors that are much more akin to a real human jump.

% By the sixth iteration, the behavior became more pronounced. 

For quantitative evaluation, we adopt human votes for the quantitative evaluation on \textit{HumanoidJump} task. As a baseline, we use \textbf{Batch Scaling (LLM-based BoN)}, which generates 30 reward functions with the same total token budget as \textbf{3D Scaling (Human Judge)}, and the best reward function is chosen from them by employing GPT-4o. 

\begin{wraptable}{r}{0.45\textwidth}
    % \vspace{-4mm}
    \centering
    \caption{Human preferences over different agents.}
    \vspace{-2mm}
    \label{tab:human_preference}
    \small
    \renewcommand{\arraystretch}{1.1}
    \begin{tabular}{lc}
        \toprule
        \textbf{Method} & \textbf{Vote} \\ 
        \midrule
        Batch Scaling (LLM-based BoN) & 3 / 20 \\
        3D Scaling (Human Judge) & \textbf{17 / 20} \\
        \bottomrule
    \end{tabular}
    \vspace{-4mm}
\end{wraptable}

To compare the performance of the two methods, we recruited 20 volunteers. Each volunteer indicated their preference between two videos presented in random order—one generated by \textbf{3D Scaling (Human Judge)} and the other by \textbf{Batch Scaling (LLM-based BoN)}. As shown in Table~\ref{tab:human_preference}, 17 out of 20 participants preferred the \textbf{3D Scaling (Human Judge)} agent, demonstrating that \textbf{3D Scaling (Human Judge)} produces behaviors more aligned with human preferences.

%% file: tables/imo.tex
\begin{table}[htbp]
\small
\centering
\caption{Average accuracy of different test-time scaling methods on IMO 2025. For each problem column, each entry in the table is shown as $n/m$, where $n$ is the number of correct trials and $m$ is the total number of trials. The final column reports the overall accuracy across all problems.}
\label{tab:imo}
% \vspace{-2mm}
\vspace{2mm}
% \resizebox{1.0\textwidth}{!}
{
\begin{tabular}{cccccccccc}
\toprule
Method & IMO1 & IMO2 & IMO3 & IMO4 & IMO5 & IMO6 & Average \\
\midrule
Context Scaling & 1/5 & 0/5 & 1/5 & 4/5 & 2/5 & 0/5 & 26.7\%\\
Batch Scaling (Vote) & 3/5 & / & 1/5 & \textbf{5/5} & 3/5 & 0/5 & 48\%\\
Turn Scaling (Reflection) & 3/5 & 2/5 & 3/5 & \textbf{5/5} & 3/5 & 0/5  & 53.3\%\\
3D Scaling (LLM Judge) & 3/5 & 2/5 & 4/5 & \textbf{5/5} & \textbf{4/5} & 0/5 
& 60\%\\
3D Scaling (Human Judge) & \textbf{5/5} & \textbf{3/5} & \textbf{5/5} & \textbf{5/5} & 3/5 & 0/5 & \textbf{70\%}\\
\bottomrule
\end{tabular}
}
\vspace{-2mm}
\end{table}

%% file: tables/cpho.tex
\begin{table}[htbp]
\small
\centering
\caption{Average accuracy of different test-time scaling methods on CPHO 2022. For each problem column, each entry in the table is shown as $n/m$, where $n$ is the number of correct trials and $m$ is the total number of trials. The final column reports the overall accuracy across all problems.}

\vspace{2mm}
\label{tab:cpho}
% \resizebox{1.0\textwidth}{!}
{
\begin{tabular}{cccccccccc}
\toprule
Method & CPHO1 & CPHO2 & CPHO3 & CPHO4 & CPHO5 & CPHO6 & Average\\
\midrule
Context Scaling & 4/5 & \textbf{5/5} & 0/5 & 0/5 & 2/5 & \textbf{5/5} & 53.3\%\\
Batch Scaling (Vote) & 4/5 & \textbf{5/5} & 0/5 & 0/5 & 4/5 & \textbf{5/5} & 60\%\\
3D Scaling (LLM Judge) & \textbf{5/5} & \textbf{5/5} & 0/5 & 0/5 & 4/5 & \textbf{5/5} & 63.3\%\\
3D Scaling (Human Judge) & \textbf{5/5} & \textbf{5/5} & \textbf{1/5} & 0/5 & \textbf{5/5} & \textbf{5/5} &
\textbf{70\%}\\
\bottomrule
\end{tabular}
}
\vspace{-2mm}
\end{table}

%% file: tables/ioi.tex
\begin{table}[htbp]
\small
\centering
\caption{Average scores of different test-time scaling methods on IOI 2025. Standard deviations are shown in parentheses.}
\label{tab:ioi}
% \vspace{-2mm}
\vspace{2mm}
% \resizebox{1.0\textwidth}{!}
{
\begin{tabular}{ccccccccc}
\toprule
Method & IOI1 & IOI2 & IOI3 & IOI4 & IOI6 & Sum \\
\midrule
Context Scaling & 2.3(2.68) & 27.49(25.0) & 8.4(4.15) & 25.6(19.3) & 13.2(17.3) & 76.99 \\
Batch Scaling (LLM-based BoN) & 3.6(3.51) & 24.19(12.7) & 9.6(3.36) & 39.6(10.4) & 16.4(10.9) & 93.39 \\
Batch Scaling (Scoring-based BoN) & 9.8(5.6) & 59.14(18.8) & \textbf{32.0(21.4)} & 57.2(12.4) & 26.6(5.2) & 184.74 \\
3D Scaling (LLM Judge) & 5.4(2.19) & 60.51(14.2) & 18.0(2.74) & 63.8(12.2) & 24.0(0.63) & 171.71 \\
3D Scaling (Human Judge) & \textbf{12.6(6.85)} & \textbf{75.13(1.16)} & 25.6(16.4) & \textbf{65.4(1.34)} & \textbf{42.8(14.8)} & \textbf{221.53} \\
\bottomrule
\end{tabular}}
\vspace{-2mm}
\end{table}

%% file: tables/Isaacgym.tex
\begin{table}
\centering
\caption{Average ground truth rewards of different test-time scaling methods on IsaacGym Tasks. The values in parentheses represent the standard deviation.}
% \vspace{-2mm}
\vspace{2mm}
\resizebox{1.0\textwidth}{!}{
\begin{tabular}{cccccccccc} 
\toprule
         & Cart.  & Ball.  & Quad.              & Ant         & Human.            & Shadow       & Allegro        \\ 
\midrule
Context Scaling & \textbf{499(0)} & \textbf{499(0)} & -0.356(0.29) & 5.262(2.49) & 6.157(0.86)  & 6.605(2.95) & 15.500(9.34)   \\
Batch Scaling (Score-based BoN) & \textbf{499(0)} & \textbf{499(0)}  & -0.0410(0.32)  & 9.350(2.34) & 8.306(1.63) & 9.476(2.44)  & 23.876(7.91)   \\
3D Scaling (Human Judge)  & \textbf{499(0)} & \textbf{499(0)}  & \textbf{-0.0183(0.29)}  & \textbf{11.142(0.37)} & \textbf{8.392(0.53)}  & \textbf{10.740(0.92)} & \textbf{24.134(6.52)} \\
\bottomrule
\end{tabular}
}
\label{tab: isaac}
\vspace{-4mm}
\end{table}

%% file: 06_conclu.tex
In this work, we revisited test-time enhancement techniques for reasoning models from the perspective of scaling laws. By unifying existing approaches under the framework of multi-dimensional test-time scaling, we identified three orthogonal axes—context, batch, and turn—each of which independently exhibits a clear scaling law. Building on this observation, we introduced 3D test-time scaling, which integrates all three dimensions to substantially extend the effective capacity of test-time compute. Our experiments demonstrated that this unified framework not only improves reasoning accuracy on challenging benchmarks such as IOI, IMO, and CPHO, but also naturally supports a human-in-the-loop paradigm that further amplifies model performance. Moreover, we showed that the same principles can be applied to embodied learning, enabling reasoning models to discover novel behaviors for humanoid robot control.

Despite these advances, important open questions remain. While our study has revealed three fundamental scaling dimensions, the capacity of test-time compute is still bottlenecked by architectural and computational constraints. It remains unclear whether additional dimensions of scaling exist that could further unlock the reasoning potential of large models. Exploring such new axes—beyond context, batch, and turn—represents an exciting direction for future research.

%% file: 07_appendix.tex
% We would suggest visiting \url{https://sites.google.com/view/few-shot-icpl/home} for more information and videos.

% \section{Full Results}
% \label{app:full-results}
% This section presents the complete results of the scaling methods introduced in Section 4 on the Mathematical Olympiad (IMO2025), Physics Olympiad (CPHO2022), Coding (IOI2025), and Embodied (IsaacGym) benchmarks, as shown in Tables \ref{tab:imo}, \ref{tab:cpho}, \ref{tab:ioi}, and \ref{tab: isaac}.

% \input{tables/imo.tex}
% \input{tables/cpho.tex}
% \input{tables/ioi.tex}
% \input{tables/Isaacgym.tex}

\section{Proofs}

\label{app: proof1}

\begin{proof}[proof of Theorem 1]
Define the probabilities of generating the correct answer $a^*$ and an incorrect answer $\tilde{a}$ as
\[
p(a^*) := \Pr_{y\sim\pi_\theta(\cdot|x)}\!\left(\text{Extract}(y) = a^*\right),
\qquad
p(\tilde{a}) := \Pr_{y\sim\pi_\theta(\cdot|x)}\!\left(\text{Extract}(y) = \tilde{a}\right).
\]

Let the extracted answers $a_1,\ldots,a_B$ be i.i.d.\ samples from the model's output distribution.
Define the count variables
\[
N(a^*) := \#\{\, i \in [B] : a_i = a^* \,\},
\qquad
N(\tilde{a}) := \#\{\, i \in [B] : a_i = \tilde{a} \,\}.
\]

By the law of large numbers,
\[
\lim_{B\rightarrow+\infty}\frac{N(a^*)}{B} \xrightarrow[]{} p(a^*),
\qquad
\lim_{B\rightarrow+\infty}\frac{N(\tilde{a})}{B} \xrightarrow[]{} p(\tilde{a}).
% \quad\text{as } B\to\infty.
\]

Therefore, the probability that the correct answer receives more votes than the incorrect answer is
\[
\Pr\!\left(N(a^*) > N(\tilde{a})\right)
=
\Pr\!\left(
\frac{N(a^*)}{B} - \frac{N(\tilde{a})}{B} > 0
\right).
\]

As $B \to \infty$, this probability converges to
\[
\Pr\!\left(N(a^*) > N(\tilde{a})\right)
\to
\begin{cases}
1, & p(a^*) > p(\tilde{a}), \\[6pt]
0, & p(a^*) < p(\tilde{a}), \\[6pt]
\frac{1}{2}, & p(a^*) = p(\tilde{a}).
\end{cases}
\]

As a result, if $p(a^*) < p(\tilde{a})$, the majority voting mechanism outputs $\tilde{a}$ instead of $a^*$:
\[
\lim_{B \to \infty} \Pr_{y_1, \dots, y_B \sim \pi_{\theta}(\cdot \mid x)} \left[ \mathcal{A}_{\mathrm{Vote}}(\{y_1, \ldots, y_B\}) = a^* \right] = 0.
\]
Thus, majority voting fails to recover the optimal candidate.

% As a result, if $p(a^*)<p(\tilde{a})$, then major vote would output $\tilde{a}$ rather than $a^*$.
% \[
% \lim_{B\to\infty}
% \Pr_{y_1,\dots,,y_B\sim \pi_{\theta}(\cdot\mid x)}\!\left[\mathcal{A}_{\mathrm{Vote}}(\{y_1,\ldots,y_B\}) = a^*\right] = 0.
% \]
% Major vote fails!
% Define the induced answer distribution
% \[
% p(a) := \Pr_{y\sim\pi_\theta(\cdot|x)}\!\left(\text{Extract}(y) = a\right).
% \]
% \[
% p(\tilde{a}) := \Pr_{y\sim\pi_\theta(\cdot|x)}\!\left(\text{Extract}(y) = a\right).
% \]
% Since the extracted answers $a_1,\ldots,a_B$ are i.i.d. samples from $\{p_a\}$, 
% the law of large numbers implies that
% \[
% \frac{|\{i : a_i = a\}|}{B} \xrightarrow[B\to\infty]{a.s.} p_a.
% \]
% Thus the selected answer converges almost surely to $\arg\max_{a} p_a$.
% If a distractor answer $\tilde{a}$ satisfies $p_{\tilde{a}} > p_{a^*}$, then 
% \[
% \arg\max_{a} p_a \not= a^*,
% \]
% and thus
% \[
% \lim_{B\to\infty}
% \Pr\!\left(\mathcal{A}_{\mathrm{Vote}}(\{y_1,\ldots,y_B\}) = a^*\right)
% = 0.
% \]
% Hence increasing $B$ decreases the expected accuracy.
\end{proof}
\section{Experiments Details}

\subsection{Benchmark choice}

\label{app:cpho_vs_ipho}

The CPHO dataset was selected over the IPHO dataset(\cite{IPHO} primarily because IPHO problems
are typically decomposed into a large number of weakly related sub-questions (e.g., 20 per problem), making it computationally expensive to evaluate the quality of each individual response. In
contrast, CPHO problems contain fewer sub-questions (e.g., 5 per problem) and exhibit strong logical coherence across all parts of a given problem. As a result, the correctness of the last sub-question
can serve as a reliable indicator of whether the model has successfully solved the entire problem.

\subsection{Extra Analysis on CPHO}

\label{app: pho}
Regarding specific accuracy rates, the LLM exhibits the following characteristics when solving physics competition problems:

\begin{enumerate}
    \item \textbf{Difference in Evaluation between Physics and Mathematics Problems:} 
    Unlike mathematics problems, where the final answer may be relatively straightforward to conjecture while the reasoning process can be highly complex, physics problems typically feature a final answer that is difficult to obtain. Consequently, if a correct final answer is produced, it generally indicates a valid reasoning process. As a result, the evaluation of physics solutions relies almost entirely on the correctness of the final answer.

    \item \textbf{Instability of Final Answers during Self-Improvement:} 
    In contrast to mathematics problems, during self-improvement iterations, the model exhibits a higher tendency to alter the final answer, reflecting greater uncertainty or refinement in the solution process for physics questions.
\end{enumerate}

\subsection{Human Evaluation Details}

\label{app:human-evaluation}

We conducted human-in-the-loop experiments with human participants. During each iteration, human evaluators select the optimal and most deficient solutions among these candidates based on whether they satisfy the task objectives and whether they can be further improved. 

The human evaluators are three volunteers, each of whom has won a gold medal in a national-level Olympics competition in mathematics, physics, or  informatics. Only the best and worst solutions themselves are fed back to the LLM to guide further self-refinement; evaluators do not provide any information about the reasons for their choices or about bugs in these responses. 

During the human evaluation process, the annotators were provided with the standard answers to the mathematics and physics problems. The evaluation protocol was as follows: annotators first assessed whether the final answer provided in the model's response was correct. Only if the final answer was correct did they proceed to evaluate the reasonableness of the key steps within the solution process.

Given the strong interdependence between subproblems within the CPHO physics problems, we manually identified and tagged the final logical step of each problem as a \textit{key subproblem}. In the system prompt, the model was explicitly instructed to present its response to this key subproblem at the very beginning of its overall reply. This design allows human annotators to quickly gauge the problem's overall correctness; if the answer to the key subproblem is correct, it serves as a strong indicator that the entire problem has likely been solved correctly.

For IOI tasks, the evaluators additionally compile and run the code generated by the model, testing it against test cases that satisfy problem-specific subtasks and constraints.

\section{Extra Experiments on IsaacGym tasks}
\label{app: isaac}

In this section, we discussed the details of experiments on IsaacGym tasks.

\subsection{Environment Details} \label{app: env}
In Table \ref{tab: app-env}, we present the observation and action dimensions, along with the task description and task metrics for 9 tasks in IsaacGym.

\input{tables/env}

\subsection{Task Metrics}
\label{app: isaac-metrics}
We employed the average of the sparse rewards across parallel environments as the task metrics, following the original setting in IsaacGym.

To assess the generated reward function in each RL run, we take the maximum task metric value sampled at fixed intervals, referred to as the \textit{task score of the reward function} (RTS). In each iteration, 3D Scaling generates $B = 6$ RL runs and selects the best and worst reward functions in that iteration. 3D Scaling performs $T = 5$ iterations and then chooses the best reward function from the final iteration as the final reward function. The RTS of this reward function is recorded as the \textit{task score} (TS) for each experiment. Due to the inherent randomness of LLMs, we conduct five experiments for all methods and report the highest TS as the \textit{final task score} (FTS) for each approach. A higher FTS indicates better overall performance across all tasks. 

\subsection{3D Scaling with Proxy Judge}
In IsaacGym tasks, it is difficult for an LLM to evaluate the quality of reward functions from videos as humans do. To address this, we use human-designed expert rewards as a proxy for human preference, enabling rapid and quantitative evaluation of our approach. This proxy represents a noise-free case that is likely easier than real human trials. Importantly, these human-designed rewards are used solely to automate sample selection and are never included in the prompts sent to the LLM; the LLM never observes the functional form of the ground-truth rewards nor receives any values from them. The results are referred to as \textbf{3D Scaling(Proxy Judge)} in the tables. We then provide the final average FTS with this extra variant in Table \ref{tab: isaac-full}
\input{tables/Isaacgym-full.tex}.

\subsection{Improvement Analysis}

As observed, on average, \textbf{3D Scaling(Proxy Judge)} achieves a 27.4\% improvement over \textbf{Batch-Scaling(Scoring-based BoN)}. 
We can also observe that 3D Scaling exhibits lower variance than Batch Scaling, indicating more stable reward learning behavior. 

While it is possible that the LLMs could generate an optimal reward function in a zero-shot manner, the primary focus of our analysis is not solely on absolute performance values. Rather, we emphasize whether 3D Scaling is capable of enhancing performance through the iterative incorporation of preferences. 
We calculated the average RTS improvement compared to the first iteration for the two tasks with the largest improvements compared with \textbf{Batch-Scaling(Scoring-based BoN)}, \textit{Ant}, and \textit{ShadowHand}. 
As shown in Fig. \ref{fig: impro}, RTS demonstrates improved performance after multiple iterations (e.g., 5 vs. 1), highlighting its effectiveness in refining reward functions.

\begin{figure}
	\centering
	% \vspace{-12mm}
	% \hspace{-2mm}
 \includegraphics[width=0.5\textwidth]{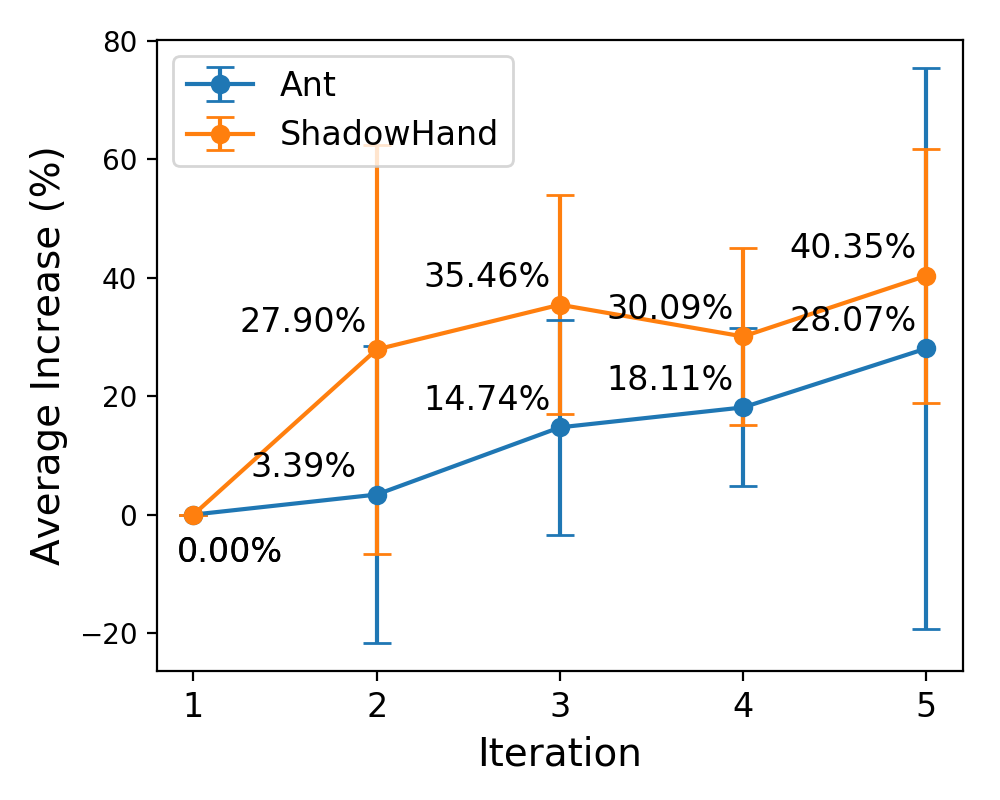}
	\centering 
	\vspace{-4mm}
    \caption{Average improvement of the Reward Task Score (RTS) compared with the first iteration in 3D scaling-Proxy Judge for the Ant and ShadowHand tasks, demonstrating the method's effectiveness in refining reward functions.}
	\vspace{-4mm}
    \label{fig: impro}
\end{figure}

\subsection{Pseudocode}

The full pseudocode of 3D Scaling on embodied AI tasks is listed in Algo. \ref{algo: app-icpl}.

\begin{algorithm}[htbp]

\SetAlgoLined
\SetKwFunction{History}{History} 
\SetKwFunction{Feedback}{Feedback} 
\SetKwProg{Fn}{Function}{:}{}
\KwIn{\# iterations $N$, \# samples in each iterations $K$, environment $\texttt{Env}$, coding LLM $\texttt{LLM}_{RF}$, difference LLM $\texttt{LLM}_{Diff}$}
\Fn{\Feedback{$\texttt{Env}, \texttt{RF}$}}{
    \KwRet The values of each component that make up \texttt{RF} during the training process in \texttt{Env}
}
\Fn{\History{$\texttt{RFlist}, \texttt{Env}, \texttt{LLM}_{Diff}$}}{
    $\texttt{HistoryFeedback} \gets ``"$
    % \texttt{HistoryFeedback} \gets \texttt{Feedback}{(\texttt{Env}, \texttt{RFlist[0]})}$
    
    \For{$i \gets 1$ to \textbf{len}$(\texttt{RFlist})-1$}{
    
    \tcp{The reward trace of historical reward functions}
    $\texttt{HistoryFeedback} \gets \texttt{HistoryFeedback} + \texttt{Feedback}{(\texttt{Env}, \texttt{RFlist[$i-1$]})}$

    \tcp{The differences between historical reward functions}
    $\texttt{HistoryFeedback} \gets \texttt{HistoryFeedback} + \texttt{LLM}_{Diff}(\texttt{DifferencePrompt}+ \texttt{RFlist[$i$]}+ \texttt{RFlist[$i-1$]})$
    }
    \KwRet{$\texttt{HistoryFeedback}$}
}

%\tcp{Use environment as context to initialize prompt}
\tcp{Initialize the prompt containing the environment context and task description}
$\texttt{Prompt} \gets \texttt{InitializePrompt}$

$\texttt{RFlist} \gets []$

\For{$i \gets 1$ to $N$}{
    $\texttt{RF}_1, \dots, \texttt{RF}_K\gets \texttt{LLM}_{RF}(\texttt{Prompt}, K)$
    
    \While{any of $\texttt{RF}_1, \dots, \texttt{RF}_K$ is not executable}
    {$j_1, \dots, j_{K'} \gets$ \text{Index of non-executable reward functions}
    
    \tcp{Regenerate non-executable reward functions}
    $\texttt{RF}_{j_1}, \dots, \texttt{RF}_{j_K'}\gets \texttt{LLM}_{RF}(\texttt{Prompt}, K')$ 
    }
    
    \tcp{Render videos for sampled reward functions}
    $\texttt{Video}_1, \dots, \texttt{Video}_K\gets \texttt{Render}(\texttt{Env}, \texttt{RF}_1), \dots, \texttt{Render}(\texttt{Env}, \texttt{RF}_K)$
    
    \tcp{Human selects the most preferred and least preferred videos}
    $G, B \gets \texttt{Human}(\texttt{Video}_1, \dots, \texttt{Video}_K)$
    
    $\texttt{GoodRF}, \texttt{BadRF} \gets \texttt{RF}_{G}, \texttt{RF}_{B}$
    
    $\texttt{RFlist}.\textbf{append}(\texttt{GoodRF})$

    \tcp{Update prompt for feedback}
    $\texttt{Prompt} \gets \texttt{GoodRF}+\texttt{Feedback}{(\texttt{Env}, \texttt{GoodRF})} + \texttt{BadRF}+\texttt{Feedback}{(\texttt{Env}, \texttt{BadRF})} + \texttt{PreferencePrompt}$

    $\texttt{Prompt} \gets \texttt{Prompt} + \texttt{History}(\texttt{RFlist}, \texttt{Env}, \texttt{LLM}_{Diff})$
}
\caption{\name}
\label{algo: app-icpl}
\end{algorithm}

\subsection{Example} 
\label{app: imp}

We use a trial of the \textit{Humanoid} task to illustrate how {\name} progressively generated improved reward functions over successive iterations. The task description is ``to make the humanoid run as fast as possible''. Throughout five iterations, adjustments were made to the penalty terms and reward weightings. In the first iteration, the total reward was calculated as $0.5 \times \text{speed\_reward} + 0.25 \times \text{deviation\_reward} + 0.25 \times \text{action\_reward}$, yielding an RTS of 5.803. The speed reward and deviation reward motivate the humanoid to run fast, while the action reward promotes smoother motion.
In the second iteration, the weight of the speed reward was increased to 0.6, while the weights for deviation and action rewards were adjusted to 0.2 each, improving the RTS to 6.113. 
In the third iteration, the action penalty was raised and the reward weights were further modified to $0.7 \times \text{speed\_reward}$, $0.15 \times \text{deviation\_reward}$, and $0.15 \times \text{action\_reward}$, resulting in an RTS of 7.915.
During the fourth iteration, the deviation penalty was reduced to 0.35 and the action penalty was lowered, with the reward weights set to 0.8, 0.1, and 0.1 for speed, deviation, and action rewards, respectively. This change led to an RTS of 8.125.
Finally, in the fifth iteration, an additional upright reward term was incorporated, with the total reward calculated as $0.7 \times \text{speed\_reward} + 0.1 \times \text{deviation\_reward} + 0.1 \times \text{action\_reward} + 0.1 \times \text{upright\_reward}$. This adjustment produced the highest RTS of 8.232, allowing {\name} to generate reward functions that were more effectively aligned with the task description. Below are the specific reward functions produced at each iteration during one experiment. 

\input{tables/example}

\section{Full Prompts}
\subsection{Full Prompts on Embodied AI tasks}
\label{app: prompt}
The prompts used in {\name} for synthesizing reward functions in Embodied AI tasks are presented in Prompts \ref{prompt: 1}, \ref{prompt: 2}, and \ref{prompt: 3}. The prompt for generating the differences between various reward functions is shown in Prompt \ref{prompt: 4}.
\input{tables/prompt}

% -------------------- System prompt subsection (for {\name}) --------------
% -------------------- System prompt subsection (for {\name}) --------------
\subsection{IMO/CPHO/IOI System Prompt}
\label{app:icpl-system-prompt}

Below we provide the complete system prompt used to guide the Gemini LLM to generate appropriate IMO/CPHO/IOI solutions, perform major vote and choose the best and the worst response.

\subsubsection*{System prompt 1}
\begin{lstlisting}[language=,caption={IMO CoT system prompt},frame=single]
-- BEGIN SYSTEM PROMPT --

"""
### Core Instructions ###

*   **Rigor is Paramount:** Your primary goal is to produce a complete and rigorously justified solution. Every step in your solution must be logically sound and clearly explained. A correct final answer derived from flawed or incomplete reasoning is considered a failure.
*   **Honesty About Completeness:** If you cannot find a complete solution, you must **not** guess or create a solution that appears correct but contains hidden flaws or justification gaps. Instead, you should present only significant partial results that you can rigorously prove. A partial result is considered significant if it represents a substantial advancement toward a full solution. Examples include:
    *   Proving a key lemma.
    *   Fully resolving one or more cases within a logically sound case-based proof.
    *   Establishing a critical property of the mathematical objects in the problem.
    *   For an optimization problem, proving an upper or lower bound without proving that this bound is achievable.
*   **Use TeX for All Mathematics:** All mathematical variables, expressions, and relations must be enclosed in TeX delimiters (e.g., `Let $n$ be an integer.`).

### Output Format ###

Your response MUST be structured into the following sections, in this exact order.

*** Final Answer ***

[Your final answer here](You should provide only the final answer here, without any explanation or reasoning.)

*** Reasoning ***

**1. Summary**

Provide a concise overview of your findings. This section must contain two parts:

*   **a. Verdict:** State clearly whether you have found a complete solution or a partial solution.
    *   **For a complete solution:** State the final answer, e.g., "I have successfully solved the problem. The final answer is..."
    *   **For a partial solution:** State the main rigorous conclusion(s) you were able to prove, e.g., "I have not found a complete solution, but I have rigorously proven that..."
*   **b. Method Sketch:** Present a high-level, conceptual outline of your solution. This sketch should allow an expert to understand the logical flow of your argument without reading the full detail. It should include:
    *   A narrative of your overall strategy.
    *   The full and precise mathematical statements of any key lemmas or major intermediate results.
    *   If applicable, describe any key constructions or case splits that form the backbone of your argument.

**2. Detailed Solution**

Present the full, step-by-step mathematical proof. Each step must be logically justified and clearly explained. The level of detail should be sufficient for an expert to verify the correctness of your reasoning without needing to fill in any gaps. This section must contain ONLY the complete, rigorous proof, free of any internal commentary, alternative approaches, or failed attempts.

### Self-Correction Instruction ###

Before finalizing your output, carefully review your "Method Sketch" and "Detailed Solution" to ensure they are clean, rigorous, and strictly adhere to all instructions provided above. Verify that every statement contributes directly to the final, coherent mathematical argument.

"""

-- END SYSTEM PROMPT --
\end{lstlisting}

\subsubsection*{System prompt 2}
\begin{lstlisting}[language=,caption={Iterative refinement in 3D Scaling system prompt},frame=single]
-- BEGIN SYSTEM PROMPT --

"""You are an expert problem solver.
Your task is to carefully read the problem statement and reflect on two previous solutions.
- previous_output1 is a relatively better attempt, but it may contain mistakes or gaps.
- previous_output2 is a weaker attempt, which might include irrelevant reasoning or errors.

Your job:
1. Identify the strengths and weaknesses of both solutions.
2. Combine the strengths and correct the weaknesses.
3. Produce a new, improved solution that is clearer, more accurate, and better structured.

Make sure the final answer is complete and stands alone as a polished solution."""

-- END SYSTEM PROMPT --

-- BEGIN QUESTION PROMPT --
"""
Problem Statement:
{problem_statement}

Better Attempt (previous_output1):
{previous_output1}

Weaker Attempt (previous_output2):
{previous_output2}
"""

-- END QUESTION PROMPT --
\end{lstlisting}

\subsubsection*{System prompt 3}
\begin{lstlisting}[language=,caption={Batchsize Comparition system prompt},frame=single]
-- BEGIN SYSTEM PROMPT --

"""
You are an expert judge. You will be given a problem statement and a numbered list of candidate solutions.
Your task is to select the single best solution and output only its 0-based index (an integer between 0 and N-1), with no extra text or explanation.
Judge by accuracy first, then completeness and clarity. If multiple are equally good, pick one deterministically (prefer lower index).
Output must be exactly one integer and nothing else.
"""

-- END SYSTEM PROMPT --

-- BEGIN QUESTION PROMPT --

"""
Problem statement:
{problem_statement}
Candidates:
{results}
Please output the 0-based index of the single best candidate.
""" 

-- END QUESTION PROMPT --

\end{lstlisting}

\subsubsection*{System prompt 4}
\begin{lstlisting}[language=,caption={Majority Vote system prompt},frame=single]
-- BEGIN SYSTEM PROMPT --

"""
You are a professional mathematical answer consistency expert. Your task is to analyze a set of mathematical answers, identify answers that are essentially the same, and find the most frequently occurring answer(s) (the mode).

# Core Principles
The criterion for judging whether two answers are the same is: whether they are mathematically equivalent, not whether the strings are exactly the same.

# Equivalence Rules
1. **Numerical equivalence**: 0.75 = 3/4 = 75% = \\frac{3}{4} = "three quarters"
2. **Algebraic expression equivalence**: 2x + 3 = 3 + 2x = (4x + 6)/2
3. **Set equivalence**: {1, 2, 3} = {3, 2, 1} = {x | x \\in {1,2,3}}
4. **Interval equivalence**: (0,1) = {x | 0 < x < 1} = "open interval from 0 to 1"
5. **Function equivalence**: f(x) = x^{2} = x*x = x^2
6. **Geometric equivalence**: "right triangle" = "triangle with a 90 degree angle"
7. **Logical equivalence**: true = "correct"

# Handling natural language answers
For answers containing explanations, extract the core mathematical content:
- "The answer is 3/4 because..." -> extract "\\frac{3}{4}"
- "I think it should be 2\\pi" -> extract "2\\pi"
- "The area of this triangle is 12 square centimeters" -> extract "12"

# Output requirements
1. **Return only the mode answer(s)**, no explanation
2. **Return in the most concise standard form** (prefer mathematical symbols)
3. **If there are multiple modes** (same highest frequency), separate them with commas
4. **Keep original format**: if it's a set, return in set form; if interval, return interval form

# Examples
Input: ["0.75", "3/4", "75%", "The answer is three quarters"]
Output: \\frac{3}{4}

Input: ["{1,2,3}", "{3,1,2}", "set contains 1,2,3"]
Output: {1,2,3}

Input: ["(0,\\infty)", "x>0", "positive real numbers"]
Output: (0,\\infty)

Input: ["2", "2.0", "two", "The answer is 2"]
Output: 2
"""

-- END SYSTEM PROMPT --
\end{lstlisting}

% \subsection*{System prompt 5}
% \begin{lstlisting}[language=,caption={CPHO CoT system prompt},frame=single]
% -- BEGIN SYSTEM PROMPT --

% -- END SYSTEM PROMPT --
% \end{lstlisting}

\subsubsection*{System prompt 5}
\begin{lstlisting}[language=,caption={CPHO CoT system prompt},frame=single]
-- BEGIN SYSTEM PROMPT --

"""
You are a professional physicist with expertise in solving high school and undergraduate level physics problems. Your task is to provide a complete, rigorous, and well-justified solution to the given physics problem.

### Core Instructions ###

*   **Complete Coverage is Paramount:** Your primary goal is to produce a complete and rigorously justified solution for every sub-question (each marked with `\item` in the problem statement). You must answer all sub-questions in the order they are presented. Do not skip any sub-question or terminate early after answering only a subset. Each sub-question's solution must be logically sound, physically accurate, and clearly explained.
*   **Rigor and Detail:** For each sub-question, provide a step-by-step detailed process that includes all reasoning, calculations, and physical principles applied. All mathematical variables, expressions, and relations must be enclosed in TeX delimiters (e.g., `$F = ma$`). Ensure that units, dimensions, and significant figures are handled appropriately where relevant.
*   **Honesty About Completeness:** If you cannot solve a sub-question completely, you must not guess or create an answer that appears correct but contains flaws. Instead, present any partial results you can rigorously justify, and clearly indicate which sub-question remains unsolved or partially solved. A partial result should represent a substantial advancement, such as deriving a key equation or setting up a correct problem framework.
*   **Final Answers Listing:** After completing the detailed solutions for all sub-questions, you must list all final answers in order at the very end of your response. This listing should include only the answers (e.g., numerical values, expressions, or conclusions), without the detailed processes.
*   **Please notice:** If there is a sub-question marked as "key sub-question", the final answer to that sub-question should be highlighted as the "Key Final Answer" in your final answers listing. If there is not such a sub-question, please treat the last sub-question as the key one. Your output should follow the structure below.

### Output Format ###

Your response MUST be structured into the following sections, in this exact order.

*** Key Final Answer ***
[The Key Final Answer]
(In this section, provide the final answer of the key sub-question only. In the problem statement part, there would be a sub-question marked as "key sub-question". 
If there is no such sub-question, please list the final answer of the last sub-question in this section.)

*** All Final Answers ***
List all final answers in order, corresponding to each sub-question. This section should be concise and contain only the answers, formatted as:

*   Sub-question 1: [Answer]
*   Sub-question 2: [Answer]
*   ... and so on for all sub-questions.

*** Reasoning ***

Present the full, step-by-step solutions for each sub-question in sequence. For each sub-question:

*   Start with a clear heading indicating the sub-question number or label (e.g., "**Sub-question 1:**").
*   Provide a rigorous and detailed solution, including all reasoning, calculations, and explanations. Use TeX for mathematics.
*   Ensure that each step is justified physically and mathematically. If a sub-question builds on previous answers, reference them appropriately.
*   Do not include commentary on alternative approaches or failed attempts-only the coherent argument for each sub-question.

### Self-Correction Instruction ###

Before finalizing your output, carefully review your response to ensure:
- All sub-questions have been addressed in the order presented, with no omissions.
- Each detailed solution is complete, rigorous, and free of gaps.
- The final answers are accurately derived and listed correctly at the end.
- The output adheres strictly to this format and instructions.

-- END SYSTEM PROMPT --
\end{lstlisting}

\subsubsection*{System prompt 6}
\begin{lstlisting}[language=,caption={IOI CoT system prompt},frame=single]
-- BEGIN SYSTEM PROMPT --

"""

### Core Instructions ###

*   **Rigor is Paramount:** Your primary goal is to produce a **fully correct and executable** C++ code. The code must handle all valid inputs defined in the problem statement and must explicitly deal with edge cases.  You should also provide a detailed explanation of your algorithm in your code to demonstrate your main method and why it is correct.
*   **Honesty About Completeness:** If you cannot provide a complete, correct code implementation, you must not guess or conceal flaws. Instead, present only the significant partial results that you can rigorously justify. For example:
    - A code that can solve subtasks with the highest total score, you should make sure its correct and provide its main algorithm.
    - A possible algorithm direction that can solve the whole problem although you do not implement it correctly.
    - A correct implementation of a critical function or subroutine.
*   **Rule for Function Call:** If the problem involves invoking functions that you are not required to implement, you must ensure that every invocation strictly adheres to the problem's specifications; otherwise, your code will be deemed invalid. Each invocation may alter the state of the data in ways that affect your objectives, and once made, such calls cannot be undone
*   **Use TeX for All Mathematics:** All mathematical variables, expressions, and relations in your algorithm must be enclosed in TeX delimiters (e.g., `Let $n$ be an integer.`).
*   **Code Format**: Your code should read the inputs from stdin solve the problem and write the answer to stdout (do not directly test on the sample inputs). Enclose your code within delimiters as follows. Ensure your c++ program contains the function requrired in the problem statement.\n```cpp\n// YOUR CODE HERE\n```"

### Output Format ###

Your response MUST be structured into the following sections, in this exact order.

**1. Summary**

Provide a concise overview of your findings. This section must contain two parts:

*   **a. Verdict:** State clearly whether you have found a complete solution or a partial solution.
    *   **For a complete solution:** State the final code, e.g., "I have successfully solved the problem. The final code is ..."
    *   **For a partial solution:** State the partial code you now have, e.g., "I have not found a complete solution, but I have a code that can solve subtasks with the highest total score, the code is ```cpp ... ```"
*   **b. Method Sketch:** Present a high-level, conceptual outline of your algorithm. This sketch should allow an expert to understand the main algorithm of your argument without reading the full detail.

**2. Detailed Solution**

Present the full, step-by-step explanation of your code.  
If your algorithm requires some proof on complexity or correctness, you should also provide the proof.
If your answer contains algorithms that can solve subtasks, you should also describe them.
The level of detail should be sufficient for an expert to verify the correctness of your code without needing to test it in testcase.

**3. Final Code**

Present your final code for the problem again. Place the solution inside one fenced code block (### Answer: (use the provided format with backticks)```cpp ...```").

### Self-Correction Instruction ###

Before finalizing your output, carefully review your code and algorithm.  
Fix any bugs, make sure the code is executable.

-- END SYSTEM PROMPT --
\end{lstlisting}

\subsubsection*{System prompt 7}
\begin{lstlisting}[language=,caption={IOI Batchsize Comparison system prompt},frame=single]
-- BEGIN SYSTEM PROMPT --

""" 
You are an expert in evaluating C++ programming solutions. Your task is to select the single best solution from several provided options based on the following criteria:

1. **Accuracy**: Prioritize solutions that solve the problem with the most correct answers and achieve the highest possible scores on subtasks.
2. **Completeness**: Consider solutions that handle edge cases effectively, ensure they cover all aspects of the problem and their time complexity is efficient enough.
3. **Clarity and Extensibility**: Evaluate the solution for clear, improvable code. Prefer solutions that are easy to extend and improve to cover more substasks.
4. **Algorithm Efficiency**: Prefer solutions with optimal time and space complexity that can scale well for larger inputs.

Choose the best solution based on these aspects and output the number of the solution you believe is the best. **If two solutions are equally good, select the one that is more accurate and complete**.

Your output should strictly follow these rules:
1. Output only the number of the best solution (starting from 1).
2. Do not output any reasoning, explanations, or extra text.
output format:
"Solution 1" or "Solution 2" or ... (just output one number)
Your output must be exactly the number of the best solution.
"""

-- END SYSTEM PROMPT --

-- BEGIN QUESTION PROMPT --

"""
Problem statement:
{problem_statement}
Candidates:
{results}
Please output only the number of the best solution (starting from 1):
""" 

-- END QUESTION PROMPT --

\end{lstlisting}

% If the listings package is not loaded in the main tex preamble, the following
% ensures the environment is available when this appendix is compiled standalone
% (it is safe if the package is already loaded in the main document).
\makeatletter
\@ifpackageloaded{listings}{}{\usepackage{listings}}
\makeatother

\section{Human-in-the-loop Preference on Innovative Tasks}
\label{sec: innovative-human}

\subsection{Demographic Data}
The participants in the human-in-the-loop preference experiments on Embodied AI Tasks consisted of 7 individuals aged 19 to 30, including 2 women and 5 men. Their educational backgrounds included 2 undergraduate students and 5 graduate students. The 20 volunteers recruited to evaluate the performance of different methods were aged 23 to 28, comprising 5 women and 15 men, with 3 undergraduates and 17 graduate students.

\subsection{IsaacGym Tasks}
\label{sec: app-human}

We evaluate human-in-the-loop preference experiments on tasks in IsaacGym, including \textit{Quadcopter, Humanoid, Ant, ShadowHand, and AllegroHand}. In these experiments, volunteers were limited to comparing reward functions based solely on videos showcasing the final policies derived from each reward function.

In the \textit{Quadcopter} task, humans evaluate performance by observing whether the quadcopter moves quickly and efficiently, and whether it stabilizes in the final position. For the \textit{Humanoid} and \textit{Ant} tasks, where the task description is "make the ant/humanoid run as fast as possible," humans estimate speed by comparing the time taken to cover the same distance and assessing the movement posture. However, due to the variability in movement postures and directions, estimating speed can introduce inaccuracies. In the \textit{ShadowHand} and \textit{AllegroHand} tasks, where the goal is ``to make the hand spin the object to a target orientation,'' Humans find it challenging to calculate the precise difference between the current orientation and the target orientation at every moment, even though the target orientation is displayed nearby. Nevertheless, humans still can estimate the duration of effective rotations with the target orientation in the video, thus evaluating the performance of a single spin. Since the target orientation regenerates upon being reached, the frequency of target orientation changes can also aid in facilitating the assessment of evaluating performance.

Due to the lack of precise environmental data, volunteers cannot make absolutely accurate judgments during the experiments. For instance, in the \textit{Humanoid} task, robots may move in varying directions, which can introduce biases in volunteers' assessments of speed. However, volunteers are still able to filter out extremely poor results and select videos with relatively better performance. In most cases, the selected results closely align with those derived from proxy human preferences, enabling effective improvements in task performance.

Below is a specific case from the \textit{Humanoid} task that illustrates the potential errors humans may make during evaluation and the learning process of the reward function under this assumption. The reward task scores (RTS) chosen by the volunteer across five iterations are $4.521, 6.069, 6.814, 6.363, 6.983$. 

In the first iteration, the ground-truth task scores of each policy were $0.593, 2.744, 4.520, 0.192, 2.517, 5.937$, although the volunteer was unaware of these scores. Initially, the volunteer eliminated policies 0 and 3, as the robots in those videos primarily exhibited spinning behavior. Subsequently, the volunteer assessed the speed of the remaining robots based on how quickly a specific robot moved out of the field. The volunteer correctly identified that the robots in policies 1 and 4 were slightly slower. However, due to minor differences in the movement directions of the robots in policies 2 and 5, the volunteer mistakenly selected policy 2 as the best option, incorrectly concluding that the robot in policy 2 was faster.

Thus, the reward function selected in iteration 1 consists of several key components: velocity reward, upright reward, force penalty, unnatural pose penalty, and action penalty. These components not only promote faster training, which is the primary objective, but also encourage the maintenance of an upright pose. Additionally, the function penalizes excessive force usage, extreme joint angles, and large action values to foster smoother and more controlled movements.

In subsequent iterations, the volunteer effectively identified reward functions that exhibited relatively better and worse performance outcomes. Adjustments were made to the weights of each component, and specific temperature values were introduced for each. These modifications resulted in a more balanced reward structure, ensuring that critical aspects exert a stronger influence, thereby allowing for greater control over the learning dynamics and improving the agent's performance in achieving the task. Even in Iteration 4, the volunteer did not select the reward function with the highest RTS (6.813) but instead opted for the second-highest reward function (RTS = 6.363). Nevertheless, the reward function exhibited consistent improvement during these iterations.

Here we show the full reward function during the process. 

\input{tables/real_human.tex}

\subsection{HumanoidJump Task} 
\label{app: jump}
In our study, we introduced a novel task: \textit{HumanoidJump}, with the task description being ``to make humanoid jump like a real human.'' The prompt of environment context in this task is shown in Prompt \ref{prompt: jump}.
\input{tables/jump}

\textbf{Reward functions.} We show the reward functions in a trial that successfully evolved a human-like jump: bending both legs to jump.
Initially, the reward function focused on encouraging vertical movement while penalizing horizontal displacement, high contact force usage, and improper joint movements. Over time, the scaling factors for the rewards and penalties were gradually adjusted by changing the temperature parameters in the exponential scaling. These adjustments aimed to enhance the model's sensitivity to different movement behaviors.
For example, the vertical movement reward's temperature was reduced, leading to more precise rewards for positive vertical movements. Similarly, the horizontal displacement penalty was fine-tuned by modifying its temperature across iterations, either decreasing or increasing the penalty's impact on lateral movements. The contact force penalty evolved by decreasing its temperature to penalize excessive force usage more strongly, especially in the later iterations, making the task more sensitive to leg contact forces. Finally, the joint usage reward was refined by adjusting the temperature to either encourage or discourage certain joint behaviors, with more focus on leg extension and contraction patterns.
Overall, the changes primarily revolved around adjusting the sensitivity of different components, refining the balance between rewards and penalties to better align the humanoid's behavior with the desired jumping performance.
\input{tables/rf.tex}

% \rebuttal{The detailed setup can be found in Appendix \ref{app: setup}.}

%% file: tables/env.tex
\begin{table}[htbp]
\centering
\caption{Details of IssacGym Tasks.}
\vspace{2mm}
\begin{tabular}{>{\raggedright\arraybackslash}p{12cm}} 
\toprule
% \multicolumn{1}{c}{\textbf{Tasks for Proxy Human Preference}} \\
% \midrule
\textbf{Environment (obs dim, action dim)} \\
{Task Description} \\
\textit{Task Metric} \\
\hline
\hline
\textbf{Cartpole (4, 1)} \\
To balance a pole on a cart so that the pole stays upright \\
\textit{duration} \\
\hline
\hline
\textbf{Quadcopter (21, 12)} \\
To make the quadcopter reach and hover near a fixed position \\
\textit{-cur\_dist} \\
\hline
\hline
\textbf{FrankaCabinet (23, 9)} \\
To open the cabinet door \\
\textit{1 if cabinet\_pos $>$ 0.39} \\
\hline
\hline
\textbf{Anymal (48, 12)} \\
To make the quadruped follow randomly chosen x, y, and yaw target velocities \\
\textit{-(linvel\_error + angvel\_error)} \\
\hline
\hline
\textbf{BallBalance (48, 12)} \\
To keep the ball on the table top without falling \\
\textit{duration} \\
\hline
\hline
\textbf{Ant (60, 8)} \\
To make the ant run forward as fast as possible \\
\textit{cur\_dist - prev\_dist} \\
\hline
\hline
\textbf{AllegroHand (88, 16)} \\
To make the hand spin the object to a target orientation \\
\textit{number of consecutive successes where 
current success is 1 if rot\_dist $<$ 0.1} \\
\hline
\hline
\textbf{Humanoid (108, 21)} \\
To make the humanoid run as fast as possible \\
\textit{cur\_dist - prev\_dist} \\
\hline
\hline
\textbf{ShadowHand (211, 20)} \\
To make the shadow hand spin the object to a target orientation \\
\textit{number of consecutive successes where 
current success is 1 if rot\_dist $<$ 0.1} \\
\bottomrule
\end{tabular}
\label{tab: app-env}
\end{table}

%% file: tables/Isaacgym-full.tex
\begin{table}
\centering
\caption{Average FTS of different test-time scaling methods on IsaacGym Tasks. The values in parentheses represent the standard deviation.}
\resizebox{1.0\textwidth}{!}{
\begin{tabular}{cccccccccc} 
\toprule
         & Cart.  & Ball.  & Quad.              & Ant         & Human.            & Shadow       & Allegro        \\ 
\midrule
Context Scaling & 499(0) & 499(0) & -0.356(0.29) & 5.262(2.49) & 6.157(0.86)  & 6.605(2.95) & 15.500(9.34)   \\
Batch Scaling(Scoring-based BoN) & 499(0) & 499(0) & -0.0410(0.32)  & 9.350(2.34) & 8.306(1.63) & 9.476(2.44)  & 23.876(7.91)   \\
3D Scaling(Human Judge)  & 499(0) & 499(0)  & -0.0183(0.29)  & 11.142(0.37) & 8.392(0.53)  & 10.740(0.92) & 24.134(6.52) \\
3D Scaling(Proxy Judge)   & 499(0) & 499(0) & -0.0195(0.09) & 12.04(1.69) & 9.227(0.93)  & 13.231(1.88) & 25.030(3.721)  \\
\bottomrule
\end{tabular}
}
\label{tab: isaac-full}
\end{table}

%% file: tables/example.tex
% First compute_reward function
\begin{tcolorbox}[breakable, colback=gray!5!white, colframe=gray!75!black, title=Humanoid Task: Reward Function with highest RTS (5.803) of Iteration 1]
\begin{lstlisting}[language=Python]
def compute_reward(root_states: torch.Tensor, actions: torch.Tensor) -> Tuple[torch.Tensor, Dict[str, torch.Tensor]]:
    velocity = root_states[: , 7:10]
    forward_velocity = velocity[:, 0]
    target_velocity = 5.0
    deviation_penalty = 0.5
    action_penalty = 0.1

    # Measure how fast the humanoid is going
    speed_reward = torch.exp((forward_velocity - target_velocity))

    # Penalize deviation from the target velocity
    deviation = torch.abs(forward_velocity - target_velocity)
    deviation_reward = torch.exp(-deviation_penalty * deviation)

    # Penalize actions, to induce smoother motion
    sum_actions = torch.sum(actions ** 2, dim=1)
    action_reward = torch.exp(-action_penalty * sum_actions)

    rewards = {'speed_reward': speed_reward, 'deviation_reward': deviation_reward, 'action_reward': action_reward }

    total_reward = 0.5 * speed_reward + 0.25 * deviation_reward + 0.25 * action_reward
    total_reward = total_reward.to(deviation.device)

    return total_reward, rewards
\end{lstlisting}
\end{tcolorbox}

% Second compute_reward function with highlighted differences
\begin{tcolorbox}[breakable, colback=gray!5!white, colframe=gray!75!black, title= Humanoid Task: Reward Function with highest RTS (6.113) of Iteration 2]
\begin{lstlisting}[language=Python]
def compute_reward(root_states: torch.Tensor, actions: torch.Tensor) -> Tuple[torch.Tensor, Dict[str, torch.Tensor]]:
    velocity = root_states[: , 7:10]
    forward_velocity = velocity[:, 0]
    target_velocity = 5.0
    deviation_penalty = 0.5
    action_penalty = 0.1

    # Measure how fast the humanoid is going
    speed_reward = torch.exp((forward_velocity - target_velocity))
    # Penalize deviation from the target velocity
    deviation = torch.abs(forward_velocity - target_velocity)
    deviation_reward = torch.exp(-deviation_penalty * deviation)
    # Penalize actions, to induce smoother motion
    sum_actions = torch.sum(actions ** 2, dim=1)
    action_reward = torch.exp(-action_penalty * sum_actions)

    rewards = {'speed_reward': speed_reward, 'deviation_reward': deviation_reward, 'action_reward': action_reward }

    # Adjusted weights in total reward calculation
    total_reward = 0.6 * speed_reward + 0.2 * deviation_reward + 0.2 * action_reward  # Changed weights
    total_reward = total_reward.to(deviation.device)

    return total_reward, rewards
\end{lstlisting}
\end{tcolorbox}

% Third compute_reward function with highlighted differences
\begin{tcolorbox}[breakable, colback=gray!5!white, colframe=gray!75!black, title=Humanoid Task: Reward Function with highest RTS (7.915) of Iteration 3]
\begin{lstlisting}[language=Python]
def compute_reward(root_states: torch.Tensor, actions: torch.Tensor) -> Tuple[torch.Tensor, Dict[str, torch.Tensor]]:
    velocity = root_states[: , 7:10]
    forward_velocity = velocity[:, 0]
    target_velocity = 5.0
    deviation_penalty = 0.5
    action_penalty = 0.15  # Increased action penalty

    # Measure how fast the humanoid is going
    speed_reward = torch.exp((forward_velocity - target_velocity))
    # Penalize deviation from the target velocity
    deviation = torch.abs(forward_velocity - target_velocity)
    deviation_reward = torch.exp(-deviation_penalty * deviation)
    # Penalize actions, to induce smoother motion
    sum_actions = torch.sum(actions ** 2, dim=1)
    action_reward = torch.exp(-action_penalty * sum_actions)

    rewards = {'speed_reward': speed_reward, 'deviation_reward': deviation_reward, 'action_reward': action_reward }

    # Adjusted weights in total reward calculation
    total_reward = 0.7 * speed_reward + 0.15 * deviation_reward + 0.15 * action_reward  # Changed weights
    total_reward = total_reward.to(deviation.device)

    return total_reward, rewards
\end{lstlisting}
\end{tcolorbox}

% Fourth compute_reward function with new upright penalty
\begin{tcolorbox}[breakable, colback=gray!5!white, colframe=gray!75!black, title=Humanoid Task: Reward Function with highest RTS (8.125) of Iteration 4]
\begin{lstlisting}[language=Python]
def compute_reward(root_states: torch.Tensor, actions: torch.Tensor) -> Tuple[torch.Tensor, Dict[str, torch.Tensor]]:
    velocity = root_states[:, 7:10]
    forward_velocity = velocity[:, 0]
    target_velocity = 5.0
    deviation_penalty = 0.35  # decrease penalty for deviation
    action_penalty = 0.1  # decrease penalty for action

    # Measure how fast the humanoid is going
    speed_reward = torch.exp((forward_velocity - target_velocity))
    # Penalize deviation from the target velocity
    deviation = torch.abs(forward_velocity - target_velocity)
    deviation_reward = torch.exp(-deviation_penalty * deviation)
    # Penalize actions, to induce smoother motion
    sum_actions = torch.sum(actions ** 2, dim=1)
    action_reward = torch.exp(-action_penalty * sum_actions)

    rewards = {'speed_reward': speed_reward, 'deviation_reward': deviation_reward, 'action_reward': action_reward }

    # Adjust the weights in total reward calculation
    total_reward = 0.8 * speed_reward + 0.1 * deviation_reward + 0.1 * action_reward # Changed weights
    total_reward = total_reward.to(deviation.device)

    return total_reward, rewards
\end{lstlisting}
\end{tcolorbox}

% Fourth compute_reward function with new upright penalty
\begin{tcolorbox}[breakable, colback=gray!5!white, colframe=gray!75!black, title=Humanoid Task: Reward Function with highest RTS (8.233) of Iteration 5]
\begin{lstlisting}[language=Python]
def compute_reward(root_states: torch.Tensor, actions: torch.Tensor, up_vec: torch.Tensor) -> Tuple[torch.Tensor, Dict[str, torch.Tensor]]:
    velocity = root_states[: , 7:10]
    forward_velocity = velocity[:, 0]
    target_velocity = 5.0
    deviation_penalty = 0.35  # Reduced deviation penalty
    action_penalty = 0.1  # Decreased action penalty

    # Measure how fast the humanoid is going
    speed_reward = torch.exp((forward_velocity - target_velocity))
    # Penalize deviation from the target velocity
    deviation = torch.abs(forward_velocity - target_velocity)
    deviation_reward = torch.exp(-deviation_penalty * deviation)
    # Penalize actions, to induce smoother motion
    sum_actions = torch.sum(actions ** 2, dim=1)
    action_reward = torch.exp(-action_penalty * sum_actions)
    # Reward for maintaining an upright position
    upright_penalty = 1.0  # New upright penalty for the humanoid
    upright_reward = torch.exp(-upright_penalty * (1 - up_vec[:, 2]))  # Added upright reward

    rewards = {'speed_reward': speed_reward, 'deviation_reward': deviation_reward, 'action_reward': action_reward, 'upright_reward': upright_reward }

    # Adjusted weights in total reward calculation
    total_reward = 0.7 * speed_reward + 0.1 * deviation_reward + 0.1 * action_reward + 0.1 * upright_reward  # Added upright reward to total
    total_reward = total_reward.to(deviation.device)

    return total_reward, rewards
\end{lstlisting}
\end{tcolorbox}

%% file: tables/prompt.tex
\definecolor{codegray}{rgb}{0.5,0.5,0.5}
\definecolor{backcolour}{RGB}{245,245,245}

\lstdefinestyle{mystyle}{
    backgroundcolor=\color{backcolour},   
    commentstyle=\color{magenta},
    keywordstyle=\color{blue},
    numberstyle=\tiny\color{codegray},
    numbers=none,
    basicstyle=\fontfamily{\ttdefault}
    \footnotesize,
    breakatwhitespace=false,         
    breaklines=true,                 
    keepspaces=true,    
    frame=single,
    numbersep=5pt,                  
    showspaces=false,                
    showstringspaces=false,
    showtabs=false,                  
    tabsize=2,
    classoffset=1, %
    keywordstyle=\color{violet},
    classoffset=0,
}
\lstset{style=mystyle}
\renewcommand\lstlistingname{Prompt}

\begin{lstlisting}[basicstyle=\fontfamily{\ttdefault}\scriptsize, breaklines=true,caption={Initial System Prompts of Synthesizing Reward Functions}\label{prompt: 1}]
You are a reward engineer trying to write reward functions to solve reinforcement learning tasks as effective as possible.
Your goal is to write a reward function for the environment that will help the agent learn the task described in text. 
Your reward function should use useful variables from the environment as inputs. As an example, the reward function signature can be: 
@torch.jit.script
def compute_reward(object_pos: torch.Tensor, goal_pos: torch.Tensor) -> Tuple[torch.Tensor, Dict[str, torch.Tensor]]:
    ...
    return reward, {}
Since the reward function will be decorated with @torch.jit.script, please make sure that the code is compatible with TorchScript (e.g., use torch tensor instead of numpy array). 
Make sure any new tensor or variable you introduce is on the same device as the input tensors. 
\end{lstlisting}
\begin{lstlisting}[basicstyle=\fontfamily{\ttdefault}\scriptsize, breaklines=true,caption={Feedback Prompts}\label{prompt: 2}]
The reward function has been iterated {current_iteration} rounds. 
In each iteration, a good reward function and a bad reward function are generated. 
The good reward function generated in the x-th iteration is denoted as "iterx-good", and the bad reward function generated is denoted as "iterx-bad". 
The following outlines the differences between these reward functions.

We trained an RL policy using iter1-good reward function code and tracked the values of the individual components in the reward function after every {epoch_freq} epochs and the maximum, mean, minimum values encountered:
<REWARD FEEDBACK>

The difference between iter2-good and iter1-good is: <DIFFERENCE>

<REPEAT UNTIL THE CURRENT ITERATION>

Next, the two reward functions generated in the {current_iteration_ordinal} iteration are provided.
The 1st generated reward function is as follows:
<REWARD FUNCTION>
We trained an RL policy using the 1st reward function code and tracked the values of the individual components in the reward function after every {epoch_freq} epochs and the maximum, mean, minimum values encountered:
<REWARD FEEDBACK>

The 2nd generated reward function is as follows:
<REWARD FUNCTION>
We trained an RL policy using the 2nd reward function code and tracked the values of the individual components in the reward function after every {epoch_freq} epochs and the maximum, mean, minimum values encountered:
<REWARD FEEDBACK>

The following content is the most important information.
Good example: 1st reward function. Bad example: 2nd reward function.
You need to modify based on the good example. DO NOT based on the code of the bad example.
Please carefully analyze the policy feedback and provide a new, improved reward function that can better solve the task. Some helpful tips for analyzing the policy feedback:
    (1) If the values for a certain reward component are near identical throughout, then this means RL is not able to optimize this component as it is written. You may consider
        (a) Changing its scale or the value of its temperature parameter
        (b) Re-writing the reward component 
        (c) Discarding the reward component
    (2) If some reward components' magnitude is significantly larger, then you must re-scale its value to a proper range
Please analyze each existing reward component in the suggested manner above first, and then write the reward function code. 
\end{lstlisting}
\begin{lstlisting}[basicstyle=\fontfamily{\ttdefault}\scriptsize, breaklines=true,caption={Prompts of Tips for Writing Reward Functions}\label{prompt: 3}]
The output of the reward function should consist of two items:
    (1) the total reward,
    (2) a dictionary of each individual reward component.
The code output should be formatted as a python code string: "```python ... ```".

Some helpful tips for writing the reward function code:
    (1) You may find it helpful to normalize the reward to a fixed range by applying transformations like torch.exp to the overall reward or its components
    (2) If you choose to transform a reward component, then you must also introduce a temperature parameter inside the transformation function; this parameter must be a named variable in the reward function and it must not be an input variable. Each transformed reward component should have its own temperature variable
    (3) Make sure the type of each input variable is correctly specified; a float input variable should not be specified as torch.Tensor
    (4) Most importantly, the reward code's input variables must contain only attributes of the provided environment class definition (namely, variables that have prefix self.). Under no circumstance can you introduce new input variables.
\end{lstlisting}
\begin{lstlisting}[basicstyle=\fontfamily{\ttdefault}\scriptsize, breaklines=true,caption={Prompts of Describing Differences}\label{prompt: 4}]
You are an engineer skilled at comparing the differences between two reward function code snippets used in reinforcement learning.
Your goal is to describe the differences between two reward function code snippets.
The following are two reward functions written in Python code used for the task: 
<TASK_DESCRIPTION>
The first reward function is as follows: 
<REWARD_FUNCTION>
The second reward function is as follows: 
<REWARD_FUNCTION>
Please directly describe the differences between these two codes. No additional descriptions other than the differences are required.
\end{lstlisting}

%% file: tables/real_human.tex
% First compute_reward function
\begin{tcolorbox}[breakable, colback=gray!5!white, colframe=gray!75!black, title=Humanoid Task: Reward Function chosen by volunteer with RTS (4.521) of Iteration 1]
\begin{lstlisting}[language=Python]
def compute_reward(
    velocity: torch.Tensor,
    dof_pos: torch.Tensor,
    dof_force_tensor: torch.Tensor,
    actions: torch.Tensor,
    up_vec: torch.Tensor,
    heading_vec: torch.Tensor
) -> Tuple[torch.Tensor, Dict[str, torch.Tensor]]:
    # Reward for forward velocity (run as fast as possible)
    velocity_reward = velocity[:, 0]
    
    # Encouragement for upright pose (penalize for deviation from vertical)
    upright_reward = up_vec[:, 2]
    
    # Penalize high force usage (energy efficiency)
    force_penalty = torch.sum(torch.abs(dof_force_tensor), dim=1)
    
    # Penalize unnatural joint positions (for instance, avoid extreme angles)
    unnatural_pose_penalty = torch.sum(torch.abs(dof_pos), dim=1)
    
    # Penalize for large actions (to promote smoother movement)
    action_penalty = torch.sum(torch.abs(actions), dim=1)
    
    # Normalize the rewards and penalties
    velocity_reward = torch.exp(velocity_reward) - 1
    upright_reward = torch.exp(upright_reward) - 1
    temperature = 1.0
    force_penalty = torch.exp(-force_penalty / temperature)
    unnatural_pose_penalty = torch.exp(-unnatural_pose_penalty / temperature)
    action_penalty = torch.exp(-action_penalty / temperature)
    
    # Combine the rewards and penalties into a single reward
    total_reward = (
        velocity_reward +
        0.5 * upright_reward -
        0.01 * force_penalty -
        0.01 * unnatural_pose_penalty -
        0.01 * action_penalty
    )
    
    # Return the total reward and each component for analysis
    reward_components = {
        "velocity_reward": velocity_reward,
        "upright_reward": upright_reward,
        "force_penalty": force_penalty,
        "unnatural_pose_penalty": unnatural_pose_penalty,
        "action_penalty": action_penalty
    }
    
    return total_reward, reward_components

\end{lstlisting}
\end{tcolorbox}

% Second compute_reward function with highlighted differences
\begin{tcolorbox}[breakable, colback=gray!5!white, colframe=gray!75!black, title= Humanoid Task: Reward Function chosen by volunteer with RTS (6.069) of Iteration 2]
\begin{lstlisting}[language=Python]
def compute_reward(
    velocity: torch.Tensor,
    dof_pos: torch.Tensor,
    dof_force_tensor: torch.Tensor,
    actions: torch.Tensor,
    up_vec: torch.Tensor,
    heading_vec: torch.Tensor
) -> Tuple[torch.Tensor, Dict[str, torch.Tensor]]:
    # Reward for forward velocity (run as fast as possible)
    velocity_reward = velocity[:, 0]
    velocity_temperature = 1.2  # increased slightly
    velocity_reward = torch.exp(velocity_reward / velocity_temperature) - 1
    
    # Encouragement for upright pose (penalize for deviation from vertical)
    upright_reward = up_vec[:, 2]
    upright_temperature = 0.5  # introduce a specific temperature
    upright_reward = torch.exp(upright_reward / upright_temperature) - 1
    
    # Penalize high force usage (energy efficiency)
    force_penalty = torch.sum(torch.abs(dof_force_tensor), dim=1)
    force_temperature = 0.1  # decreased to make it more significant
    force_penalty = torch.exp(-force_penalty / force_temperature)
    
    # Penalize unnatural joint positions (for instance, avoid extreme angles)
    unnatural_pose_penalty = torch.sum(torch.abs(dof_pos), dim=1)
    pose_temperature = 0.1  # decreased to make it more significant
    unnatural_pose_penalty = torch.exp(-unnatural_pose_penalty / pose_temperature)
    
    # Penalize for large actions (to promote smoother movement)
    action_penalty = torch.sum(torch.abs(actions), dim=1)
    action_temperature = 0.1  # decreased to make it more significant
    action_penalty = torch.exp(-action_penalty / action_temperature)
    
    # Combine the rewards and penalties into a single reward
    total_reward = (
        velocity_reward +
        0.5 * upright_reward -
        0.01 * force_penalty -
        0.01 * unnatural_pose_penalty -
        0.01 * action_penalty
    )
    
    # Return the total reward and each component for analysis
    reward_components = {
        "velocity_reward": velocity_reward,
        "upright_reward": upright_reward,
        "force_penalty": force_penalty,
        "unnatural_pose_penalty": unnatural_pose_penalty,
        "action_penalty": action_penalty
    }
    
    return total_reward, reward_components

\end{lstlisting}
\end{tcolorbox}

% Third compute_reward function with highlighted differences
\begin{tcolorbox}[breakable, colback=gray!5!white, colframe=gray!75!black, title= Humanoid Task: Reward Function chosen by volunteer with RTS (6.814) of Iteration 3]
\begin{lstlisting}[language=Python]
def compute_reward(
    velocity: torch.Tensor,
    dof_pos: torch.Tensor,
    dof_force_tensor: torch.Tensor,
    actions: torch.Tensor,
    up_vec: torch.Tensor,
    heading_vec: torch.Tensor
) -> Tuple[torch.Tensor, Dict[str, torch.Tensor]]:
    # Reward for forward velocity (run as fast as possible)
    velocity_reward = velocity[:, 0]
    velocity_temperature = 1.1  # minor adjustment
    velocity_reward = torch.exp(velocity_reward / velocity_temperature) - 1
    
    # Encouragement for upright pose (penalize for deviation from vertical)
    upright_reward = up_vec[:, 2]
    upright_temperature = 0.6  # slight adjustment
    upright_reward = torch.exp(upright_reward / upright_temperature) - 1
    
    # Penalize high force usage (energy efficiency)
    force_penalty = torch.sum(torch.abs(dof_force_tensor), dim=1)
    force_temperature = 0.15  # increased to try to make it effective
    force_penalty = torch.exp(-force_penalty / force_temperature)
    
    # Penalize unnatural joint positions (for instance, avoid extreme angles)
    unnatural_pose_penalty = torch.sum(torch.abs(dof_pos), dim=1)
    pose_temperature = 0.2  # increased to try to make it effective
    unnatural_pose_penalty = torch.exp(-unnatural_pose_penalty / pose_temperature)
    
    # Penalize for large actions (to promote smoother movement)
    action_penalty = torch.sum(torch.abs(actions), dim=1)
    action_temperature = 0.2  # increased to try to make it effective
    action_penalty = torch.exp(-action_penalty / action_temperature)
    
    # Combine the rewards and penalties into a single reward
    total_reward = (
        velocity_reward +
        0.5 * upright_reward -
        0.02 * force_penalty -  # increased slightly for more impact
        0.02 * unnatural_pose_penalty -  # increased slightly for more impact
        0.02 * action_penalty  # increased slightly for more impact
    )
    
    # Return the total reward and each component for analysis
    reward_components = {
        "velocity_reward": velocity_reward,
        "upright_reward": upright_reward,
        "force_penalty": force_penalty,
        "unnatural_pose_penalty": unnatural_pose_penalty,
        "action_penalty": action_penalty
    }
    
    return total_reward, reward_components

\end{lstlisting}
\end{tcolorbox}

% Fourth compute_reward function with highlighted differences
\begin{tcolorbox}[breakable, colback=gray!5!white, colframe=gray!75!black, title= Humanoid Task: Reward Function chosen by volunteer with RTS (6.363) of Iteration 4]
\begin{lstlisting}[language=Python]
def compute_reward(
    velocity: torch.Tensor,
    dof_pos: torch.Tensor,
    dof_force_tensor: torch.Tensor,
    actions: torch.Tensor,
    up_vec: torch.Tensor,
    heading_vec: torch.Tensor
) -> Tuple[torch.Tensor, Dict[str, torch.Tensor]]:
    # Reward for forward velocity (run as fast as possible)
    velocity_reward = velocity[:, 0]
    velocity_temperature = 1.05  # slight adjustment to refine the impact
    velocity_reward = torch.exp(velocity_reward / velocity_temperature) - 1
    
    # Encouragement for upright pose (penalize for deviation from vertical)
    upright_reward = up_vec[:, 2]
    upright_temperature = 0.65  # slight loosening for more upright reward
    upright_reward = torch.exp(upright_reward / upright_temperature) - 1
    
    # Penalize high force usage (energy efficiency)
    force_penalty = torch.sum(torch.abs(dof_force_tensor), dim=1)
    force_temperature = 0.2  # increased to make it more significant
    force_penalty = torch.exp(-force_penalty / force_temperature)
    
    # Penalize unnatural joint positions (for instance, avoid extreme angles)
    unnatural_pose_penalty = torch.sum(torch.abs(dof_pos), dim=1)
    pose_temperature = 0.25  # slight increase to make this component effective
    unnatural_pose_penalty = torch.exp(-unnatural_pose_penalty / pose_temperature)
    
    # Penalize for large actions (to promote smoother movement)
    action_penalty = torch.sum(torch.abs(actions), dim=1)
    action_temperature = 0.25  # slightly adjusted for more prominent constraint
    action_penalty = torch.exp(-action_penalty / action_temperature)
    
    # Combine the rewards and penalties into a single reward
    total_reward = (
        velocity_reward +
        0.5 * upright_reward -
        0.015 * force_penalty -  # slight increase for more impact
        0.015 * unnatural_pose_penalty -  # slight increase for more impact
        0.015 * action_penalty  # slight increase for more impact
    )
    
    # Return the total reward and each component for analysis
    reward_components = {
        "velocity_reward": velocity_reward,
        "upright_reward": upright_reward,
        "force_penalty": force_penalty,
        "unnatural_pose_penalty": unnatural_pose_penalty,
        "action_penalty": action_penalty
    }
    
    return total_reward, reward_components

\end{lstlisting}
\end{tcolorbox}

% Fourth best compute_reward function with highlighted differences
\begin{tcolorbox}[breakable, colback=gray!5!white, colframe=gray!75!black, title= Humanoid Task: Reward Function with best RTS (6.813) of Iteration 4(not chosen by volunteer)]
\begin{lstlisting}[language=Python]
def compute_reward(
    velocity: torch.Tensor,
    dof_pos: torch.Tensor,
    dof_force_tensor: torch.Tensor,
    actions: torch.Tensor,
    up_vec: torch.Tensor,
    heading_vec: torch.Tensor
) -> Tuple[torch.Tensor, Dict[str, torch.Tensor]]:
    # Reward for forward velocity (run as fast as possible)
    velocity_reward = velocity[:, 0]
    velocity_temperature = 1.15
    velocity_reward = torch.exp(velocity_reward / velocity_temperature) - 1
    
    # Encouragement for upright pose (penalize for deviation from vertical)
    upright_reward = up_vec[:, 2]
    upright_temperature = 0.55
    upright_reward = torch.exp(upright_reward / upright_temperature) - 1
    
    # Penalize high force usage (energy efficiency)
    force_penalty = torch.sum(torch.abs(dof_force_tensor), dim=1)
    force_temperature = 0.12
    force_penalty = torch.exp(-force_penalty / force_temperature)
    
    # Penalize unnatural joint positions (for instance, avoid extreme angles)
    unnatural_pose_penalty = torch.sum(torch.abs(dof_pos), dim=1)
    pose_temperature = 0.18
    unnatural_pose_penalty = torch.exp(-unnatural_pose_penalty / pose_temperature)
    
    # Penalize for large actions (to promote smoother movement)
    action_penalty = torch.sum(torch.abs(actions), dim=1)
    action_temperature = 0.18
    action_penalty = torch.exp(-action_penalty / action_temperature)
    
    # Combine the rewards and penalties into a single reward
    total_reward = (
        velocity_reward +
        0.5 * upright_reward -
        0.02 * force_penalty -
        0.02 * unnatural_pose_penalty -
        0.02 * action_penalty
    )
    
    # Return the total reward and each component for analysis
    reward_components = {
        "velocity_reward": velocity_reward,
        "upright_reward": upright_reward,
        "force_penalty": force_penalty,
        "unnatural_pose_penalty": unnatural_pose_penalty,
        "action_penalty": action_penalty
    }
    
    return total_reward, reward_components

\end{lstlisting}
\end{tcolorbox}

% Fifth compute_reward function with highlighted differences
\begin{tcolorbox}[breakable, colback=gray!5!white, colframe=gray!75!black, title= Humanoid Task: Reward Function chosen by volunteer with RTS (6.983) of Iteration 5]
\begin{lstlisting}[language=Python]
def compute_reward(
    velocity: torch.Tensor,
    dof_pos: torch.Tensor,
    dof_force_tensor: torch.Tensor,
    actions: torch.Tensor,
    up_vec: torch.Tensor,
    heading_vec: torch.Tensor
) -> Tuple[torch.Tensor, Dict[str, torch.Tensor]]:
    # Adjusted parameters based on analysis
    velocity_temperature = 1.1
    upright_temperature = 0.65
    force_temperature = 0.25
    pose_temperature = 0.3
    action_temperature = 0.3

    # Reward for forward velocity (run as fast as possible)
    velocity_reward = velocity[:, 0]
    velocity_reward = torch.exp(velocity_reward / velocity_temperature) - 1
    
    # Encouragement for upright pose (penalize for deviation from vertical)
    upright_reward = up_vec[:, 2]
    upright_reward = torch.exp(upright_reward / upright_temperature) - 1
    
    # Penalize high force usage (energy efficiency)
    force_penalty = torch.sum(torch.abs(dof_force_tensor), dim=1)
    force_penalty = torch.exp(-force_penalty / force_temperature)
    
    # Penalize unnatural joint positions (for instance, avoid extreme angles)
    unnatural_pose_penalty = torch.sum(torch.abs(dof_pos), dim=1)
    unnatural_pose_penalty = torch.exp(-unnatural_pose_penalty / pose_temperature)
    
    # Penalize for large actions (to promote smoother movement)
    action_penalty = torch.sum(torch.abs(actions), dim=1)
    action_penalty = torch.exp(-action_penalty / action_temperature)
    
    # Combine the rewards and penalties into a single reward
    total_reward = (
        velocity_reward +
        0.5 * upright_reward -
        0.02 * force_penalty -
        0.02 * unnatural_pose_penalty -
        0.02 * action_penalty
    )
    
    # Return the total reward and each component for analysis
    reward_components = {
        "velocity_reward": velocity_reward,
        "upright_reward": upright_reward,
        "force_penalty": force_penalty,
        "unnatural_pose_penalty": unnatural_pose_penalty,
        "action_penalty": action_penalty
    }
    
    return total_reward, reward_components

\end{lstlisting}
\end{tcolorbox}

%% file: tables/jump.tex
\begin{lstlisting}[basicstyle=\fontfamily{\ttdefault}\scriptsize, breaklines=true,caption={Prompts of Environment Context in \textit{HumanoidJump} Task}\label{prompt: jump}]
class HumanoidJump(VecTask):
    """Rest of the environment definition omitted."""
    def compute_observations(self):
        self.gym.refresh_dof_state_tensor(self.sim)
        self.gym.refresh_actor_root_state_tensor(self.sim)
        self.gym.refresh_force_sensor_tensor(self.sim)
        self.gym.refresh_dof_force_tensor(self.sim)

        self.obs_buf[:], self.torso_position[:], 
        self.prev_torso_position[:], self.velocity_world[:], 
        self.angular_velocity_world[:], self.velocity_local[:], 
        self.angular_velocity_local[:], self.up_vec[:], 
        self.heading_vec[:], self.right_leg_contact_force[:], 
        self.left_leg_contact_force[:] = \
            compute_humanoid_jump_observations(
            self.obs_buf, self.root_states, self.torso_position,
            self.inv_start_rot, self.dof_pos, self.dof_vel, 
            self.dof_force_tensor, self.dof_limits_lower, 
            self.dof_limits_upper, self.dof_vel_scale, 
            self.vec_sensor_tensor, self.actions, 
            self.dt, self.contact_force_scale, 
            self.angular_velocity_scale, 
            self.basis_vec0, self.basis_vec1)
    
    def compute_humanoid_jump_observations(obs_buf, root_states, torso_position, inv_start_rot, dof_pos, dof_vel, dof_force, dof_limits_lower, dof_limits_upper, dof_vel_scale, sensor_force_torques, actions, dt, contact_force_scale, angular_velocity_scale, basis_vec0, basis_vec1):
        # type: (Tensor, Tensor, Tensor, Tensor, Tensor, Tensor, Tensor, Tensor, Tensor, float, Tensor, Tensor, float, float, float, Tensor, Tensor) -> Tuple[Tensor, Tensor, Tensor, Tensor, Tensor, Tensor, Tensor, Tensor, Tensor, Tensor, Tensor]
    
        prev_torso_position_new = torso_position.clone()
    
        torso_position = root_states[:, 0:3]
        torso_rotation = root_states[:, 3:7]
        velocity_world = root_states[:, 7:10]
        angular_velocity_world = root_states[:, 10:13]
    
        torso_quat, up_proj, up_vec, heading_vec = compute_heading_and_up_vec(
            torso_rotation, inv_start_rot, basis_vec0, basis_vec1, 2)
    
        velocity_local, angular_velocity_local, roll, pitch, yaw = compute_rot_new(
            torso_quat, velocity_world, angular_velocity_world)
    
        roll = normalize_angle(roll).unsqueeze(-1)
        yaw = normalize_angle(yaw).unsqueeze(-1)
        dof_pos_scaled = unscale(dof_pos, dof_limits_lower, dof_limits_upper)
        scale_angular_velocity_local = angular_velocity_local * angular_velocity_scale
    
        obs = torch.cat((root_states[:, 0:3].view(-1, 3), velocity_local, 
                         scale_angular_velocity_local,
                         yaw, roll, up_proj.unsqueeze(-1),
                         dof_pos_scaled, dof_vel * dof_vel_scale, 
                         dof_force * contact_force_scale, 
                         sensor_force_torques.view(-1, 12) * contact_force_scale, 
                         actions), dim=-1)
    
        right_leg_contact_force = sensor_force_torques[:, 0:3]
        left_leg_contact_force = sensor_force_torques[:, 6:9]
        
        abdomen_y_pos = dof_pos[:, 0]
        abdomen_z_pos = dof_pos[:, 1]
        abdomen_x_pos = dof_pos[:, 2]
        right_hip_x_pos = dof_pos[:, 3]
        right_hip_z_pos = dof_pos[:, 4]
        right_hip_y_pos = dof_pos[:, 5]
        right_knee_pos = dof_pos[:, 6]
        right_ankle_x_pos = dof_pos[:, 7]
        right_ankle_y_pos = dof_pos[:, 8]
        left_hip_x_pos = dof_pos[:, 9]
        left_hip_z_pos = dof_pos[:, 10]
        left_hip_y_pos = dof_pos[:, 11]
        left_knee_pos = dof_pos[:, 12]
        left_ankle_x_pos = dof_pos[:, 13]
        left_ankle_y_pos = dof_pos[:, 14]
        right_shoulder1_pos = dof_pos[:, 15]
        right_shoulder2_pos = dof_pos[:, 16]
        right_elbow_pos = dof_pos[:, 17]
        left_shoulder1_pos = dof_pos[:, 18]
        left_shoulder2_pos = dof_pos[:, 19]
        left_elbow_pos = dof_pos[:, 20]
    
        right_shoulder1_action = actions[:, 15]
        right_shoulder2_action = actions[:, 16]
        right_elbow_action = actions[:, 17]
        left_shoulder1_action = actions[:, 18]
        left_shoulder2_action = actions[:, 19]
        left_elbow_action = actions[:, 20]
        
        return obs, torso_position, prev_torso_position_new, velocity_world, 
               angular_velocity_world, velocity_local, scale_angular_velocity_local, 
               up_vec, heading_vec, right_leg_contact_force, left_leg_contact_force
\end{lstlisting}

%% file: tables/rf.tex
% First compute_reward function
\begin{tcolorbox}[breakable, colback=gray!5!white, colframe=gray!75!black, title=HumanoidJump Task: Reward Function of Iteration 1]
\begin{lstlisting}[language=Python]
def compute_reward(torso_position: torch.Tensor, prev_torso_position: torch.Tensor, velocity_world: torch.Tensor, 
                   right_leg_contact_force: torch.Tensor, left_leg_contact_force: torch.Tensor, dof_pos: torch.Tensor) -> Tuple[torch.Tensor, Dict[str, torch.Tensor]]:
    # Ensure all tensors are on the same device
    device = torso_position.device

    # Compute vertical torso movement reward
    vertical_movement = torso_position[:, 2] - prev_torso_position[:, 2]
    vertical_movement_reward = torch.clamp(vertical_movement, min=0.0)  # Reward positive vertical movement
    vertical_movement_reward = torch.exp(vertical_movement_reward / 0.1)  # Use exponential scaling with temperature
    
    # Compute horizontal displacement penalty
    horizontal_displacement = torch.sum(torch.abs(torso_position[:, :2] - prev_torso_position[:, :2]), dim=-1)
    horizontal_displacement_penalty = torch.exp(-horizontal_displacement / 0.1)  # Penalize large movements with temperature

    # Compute leg forces usage reward
    contact_force_usage = torch.sum(torch.abs(right_leg_contact_force) + torch.abs(left_leg_contact_force), dim=-1)
    contact_force_usage_penalty = torch.exp(-contact_force_usage / 10.0)  # Penalize high contact force usage with temperature

    # Compute joint usage reward (encourages proper leg extension and contraction)
    leg_joints_indices = torch.tensor([6, 7, 8, 12, 13, 14], device=device)  # Indices of leg joints 
    leg_joint_usage = torch.mean(dof_pos[:, leg_joints_indices], dim=-1)
    leg_joint_usage_reward = torch.exp(-torch.abs(leg_joint_usage) / 0.1)  # Encourage movements from neutral position

    # Sum all rewards and penalties
    total_reward = vertical_movement_reward + horizontal_displacement_penalty + contact_force_usage_penalty + leg_joint_usage_reward

    # Create a dictionary for individual reward components
    reward_components = {
        'vertical_movement_reward': vertical_movement_reward,
        'horizontal_displacement_penalty': horizontal_displacement_penalty,
        'contact_force_usage_penalty': contact_force_usage_penalty,
        'leg_joint_usage_reward': leg_joint_usage_reward
    }
    
    return total_reward, reward_components
\end{lstlisting}
\end{tcolorbox}

% Second compute_reward function 
\begin{tcolorbox}[breakable, colback=gray!5!white, colframe=gray!75!black, title= HumanoidJump Task: Reward Function of Iteration 2]
\begin{lstlisting}[language=Python]
def compute_reward(
    torso_position: torch.Tensor,
    prev_torso_position: torch.Tensor,
    velocity_world: torch.Tensor,
    right_leg_contact_force: torch.Tensor,
    left_leg_contact_force: torch.Tensor,
    dof_pos: torch.Tensor
) -> Tuple[torch.Tensor, Dict[str, torch.Tensor]]:
    # Ensure all tensors are on the same device
    device = torso_position.device

    # Compute vertical torso movement reward
    vertical_movement = torso_position[:, 2] - prev_torso_position[:, 2]
    vertical_movement_reward = torch.clamp(vertical_movement, min=0.0)  # Reward positive vertical movement
    temperature_vertical = 0.05
    vertical_movement_reward = torch.exp(vertical_movement_reward / temperature_vertical)  # Use exponential scaling

    # Compute horizontal displacement penalty
    horizontal_displacement = torch.sum(torch.abs(torso_position[:, :2] - prev_torso_position[:, :2]), dim=-1)
    temperature_horizontal = 0.1
    horizontal_displacement_penalty = torch.exp(-horizontal_displacement / temperature_horizontal)  # Penalize large movements with temperature

    # Compute leg forces usage reward
    contact_force_usage = torch.sum(torch.abs(right_leg_contact_force) + torch.abs(left_leg_contact_force), dim=-1)
    temperature_force = 1.0
    contact_force_usage_penalty = torch.exp(-contact_force_usage / temperature_force)  # Penalize high contact force usage with temperature

    # Compute joint usage reward (encourages proper leg extension and contraction)
    leg_joints_indices = torch.tensor([6, 7, 8, 12, 13, 14], device=device)  # Indices of leg joints 
    leg_joint_usage = torch.mean(dof_pos[:, leg_joints_indices], dim=-1)
    temperature_joints = 0.05
    leg_joint_usage_reward = torch.exp(-torch.abs(leg_joint_usage) / temperature_joints)  # Encourage movements from the neutral position

    # Sum all rewards and penalties
    total_reward = vertical_movement_reward + horizontal_displacement_penalty + contact_force_usage_penalty + leg_joint_usage_reward

    # Create a dictionary for individual reward components
    reward_components = {
        'vertical_movement_reward': vertical_movement_reward,
        'horizontal_displacement_penalty': horizontal_displacement_penalty,
        'contact_force_usage_penalty': contact_force_usage_penalty,
        'leg_joint_usage_reward': leg_joint_usage_reward
    }

    return total_reward, reward_components
\end{lstlisting}
\end{tcolorbox}

% Third compute_reward function 
\begin{tcolorbox}[breakable, colback=gray!5!white, colframe=gray!75!black, title=HumanoidJump Task: Reward Function of Iteration 3]
\begin{lstlisting}[language=Python]
def compute_reward(
    torso_position: torch.Tensor,
    prev_torso_position: torch.Tensor,
    velocity_world: torch.Tensor,
    right_leg_contact_force: torch.Tensor,
    left_leg_contact_force: torch.Tensor,
    dof_pos: torch.Tensor
) -> Tuple[torch.Tensor, Dict[str, torch.Tensor]]:
    # Ensure all tensors are on the same device
    device = torso_position.device

    # Compute vertical torso movement reward
    vertical_movement = torso_position[:, 2] - prev_torso_position[:, 2]
    vertical_movement_reward = torch.clamp(vertical_movement, min=0.0)  # Reward positive vertical movement
    temperature_vertical = 0.05
    vertical_movement_reward = torch.exp(vertical_movement_reward / temperature_vertical)  # Use exponential scaling

    # Compute horizontal displacement penalty
    horizontal_displacement = torch.sum(torch.abs(torso_position[:, :2] - prev_torso_position[:, :2]), dim=-1)
    temperature_horizontal = 0.05  # Adjusted temperature for finer tuning
    horizontal_displacement_penalty = torch.exp(-horizontal_displacement / temperature_horizontal)  # Penalize large movements

    # Compute leg forces usage reward
    contact_force_usage = torch.sum(torch.abs(right_leg_contact_force) + torch.abs(left_leg_contact_force), dim=-1)
    temperature_force = 5.0  # Adjusted to make contact force usage more noticeable
    contact_force_usage_penalty = torch.exp(-contact_force_usage / temperature_force)  # Penalize high contact force usage

    # Compute joint usage reward (encourages proper leg extension and contraction)
    leg_joints_indices = torch.tensor([6, 7, 8, 12, 13, 14], device=device)  # Indices of leg joints 
    leg_joint_usage = torch.mean(dof_pos[:, leg_joints_indices], dim=-1)
    temperature_joints = 0.05
    leg_joint_usage_reward = torch.exp(-torch.abs(leg_joint_usage) / temperature_joints)  # Encourage movements from the neutral position

    # Sum all rewards and penalties
    total_reward = vertical_movement_reward + horizontal_displacement_penalty + contact_force_usage_penalty + leg_joint_usage_reward

    # Create a dictionary for individual reward components
    reward_components = {
        'vertical_movement_reward': vertical_movement_reward,
        'horizontal_displacement_penalty': horizontal_displacement_penalty,
        'contact_force_usage_penalty': contact_force_usage_penalty,
        'leg_joint_usage_reward': leg_joint_usage_reward
    }

    return total_reward, reward_components
\end{lstlisting}
\end{tcolorbox}

% Fourth compute_reward function 
\begin{tcolorbox}[breakable, colback=gray!5!white, colframe=gray!75!black, title=HumanoidJump Task: Reward Function of Iteration 4]
\begin{lstlisting}[language=Python]
def compute_reward(
    torso_position: torch.Tensor,
    prev_torso_position: torch.Tensor,
    velocity_world: torch.Tensor,
    right_leg_contact_force: torch.Tensor,
    left_leg_contact_force: torch.Tensor,
    dof_pos: torch.Tensor
) -> Tuple[torch.Tensor, Dict[str, torch.Tensor]]:
    # Ensure all tensors are on the same device
    device = torso_position.device

    # Compute vertical torso movement reward
    vertical_movement = torso_position[:, 2] - prev_torso_position[:, 2]
    vertical_movement_reward = torch.clamp(vertical_movement, min=0.0)  # Reward positive vertical movement
    temperature_vertical = 0.04  # Adjusted temperature for better sensitivity
    vertical_movement_reward = torch.exp(vertical_movement_reward / temperature_vertical)  # Use exponential scaling

    # Compute horizontal displacement penalty
    horizontal_displacement = torch.sum(torch.abs(torso_position[:, :2] - prev_torso_position[:, :2]), dim=-1)
    temperature_horizontal = 0.1  # Increased temperature to penalize horizontal movement more 
    horizontal_displacement_penalty = torch.exp(-horizontal_displacement / temperature_horizontal)  # Penalize large movements

    # Compute leg forces usage reward
    contact_force_usage = torch.sum(torch.abs(right_leg_contact_force) + torch.abs(left_leg_contact_force), dim=-1)
    temperature_force = 0.1  # Significantly increase sensitivity to contact forces
    contact_force_usage_penalty = torch.exp(-contact_force_usage / temperature_force)  # Penalize high contact force usage

    # Compute joint usage reward (encourages proper leg extension and contraction)
    leg_joints_indices = torch.tensor([6, 7, 8, 12, 13, 14], device=device)  # Indices of leg joints 
    leg_joint_usage = torch.mean(dof_pos[:, leg_joints_indices], dim=-1)
    temperature_joints = 0.02  # Adjusted for joint usage sensitivity
    leg_joint_usage_reward = torch.exp(-torch.abs(leg_joint_usage) / temperature_joints)  # Encourage movements from the neutral position

    # Sum all rewards and penalties
    total_reward = vertical_movement_reward + horizontal_displacement_penalty + contact_force_usage_penalty + leg_joint_usage_reward

    # Create a dictionary for individual reward components
    reward_components = {
        'vertical_movement_reward': vertical_movement_reward,
        'horizontal_displacement_penalty': horizontal_displacement_penalty,
        'contact_force_usage_penalty': contact_force_usage_penalty,
        'leg_joint_usage_reward': leg_joint_usage_reward
    }

    return total_reward, reward_components
\end{lstlisting}
\end{tcolorbox}

% Fifth compute_reward function 
\begin{tcolorbox}[breakable, colback=gray!5!white, colframe=gray!75!black, title=HumanoidJump Task: Reward Function of Iteration 5]
\begin{lstlisting}[language=Python]
def compute_reward(
    torso_position: torch.Tensor,
    prev_torso_position: torch.Tensor,
    velocity_world: torch.Tensor,
    right_leg_contact_force: torch.Tensor,
    left_leg_contact_force: torch.Tensor,
    dof_pos: torch.Tensor
) -> Tuple[torch.Tensor, Dict[str, torch.Tensor]]:
    # Ensure all tensors are on the same device
    device = torso_position.device

    # Compute vertical torso movement reward
    vertical_movement = torso_position[:, 2] - prev_torso_position[:, 2]
    vertical_movement_reward = torch.clamp(vertical_movement, min=0.0)  # Reward positive vertical movement
    temperature_vertical = 0.04  # Adjusted temperature for better sensitivity
    vertical_movement_reward = torch.exp(vertical_movement_reward / temperature_vertical)  # Use exponential scaling

    # Compute horizontal displacement penalty
    horizontal_displacement = torch.sum(torch.abs(torso_position[:, :2] - prev_torso_position[:, :2]), dim=-1)
    temperature_horizontal = 0.05  # Decreased temperature for more sensitivity
    horizontal_displacement_penalty = torch.exp(-horizontal_displacement / temperature_horizontal)  # Penalize large movements

    # Compute leg forces usage penalty (Rewritten to reduce contact force)
    contact_force_usage = torch.sum(torch.abs(right_leg_contact_force) + torch.abs(left_leg_contact_force), dim=-1)
    temperature_force = 0.5  # Adjusted to penalize contact force usage
    contact_force_usage_penalty = torch.exp(-contact_force_usage / temperature_force)  # Penalize high contact force usage

    # Compute joint usage reward (encourages proper leg extension and contraction)
    leg_joints_indices = torch.tensor([6, 7, 8, 12, 13, 14], device=device)  # Indices of leg joints 
    leg_joint_usage = torch.mean(torch.abs(dof_pos[:, leg_joints_indices]), dim=-1)
    temperature_joints = 0.02  # Adjusted for joint usage sensitivity
    leg_joint_usage_reward = torch.exp(-leg_joint_usage / temperature_joints)  # Encourage movements from the neutral position

    # Sum all rewards and penalties
    total_reward = vertical_movement_reward + horizontal_displacement_penalty + contact_force_usage_penalty + leg_joint_usage_reward

    # Create a dictionary for individual reward components
    reward_components = {
        'vertical_movement_reward': vertical_movement_reward,
        'horizontal_displacement_penalty': horizontal_displacement_penalty,
        'contact_force_usage_penalty': contact_force_usage_penalty,
        'leg_joint_usage_reward': leg_joint_usage_reward
    }

    return total_reward, reward_components
\end{lstlisting}
\end{tcolorbox}

% Sixth compute_reward function
\begin{tcolorbox}[breakable, colback=gray!5!white, colframe=gray!75!black, title=HumanoidJump Task: Reward Function of Iteration 6]
\begin{lstlisting}[language=Python]
def compute_reward(
    torso_position: torch.Tensor,
    prev_torso_position: torch.Tensor,
    velocity_world: torch.Tensor,
    right_leg_contact_force: torch.Tensor,
    left_leg_contact_force: torch.Tensor,
    dof_pos: torch.Tensor
) -> Tuple[torch.Tensor, Dict[str, torch.Tensor]]:
    # Ensure all tensors are on the same device
    device = torso_position.device

    # Compute vertical torso movement reward
    vertical_movement = torso_position[:, 2] - prev_torso_position[:, 2]
    vertical_movement_reward = torch.clamp(vertical_movement, min=0.0)  # Reward positive vertical movement
    temperature_vertical = 0.03  # Fine-tuned temperature for better sensitivity
    vertical_movement_reward = torch.exp(vertical_movement_reward / temperature_vertical)  # Use exponential scaling

    # Compute horizontal displacement penalty
    horizontal_displacement = torch.sum(torch.abs(torso_position[:, :2] - prev_torso_position[:, :2]), dim=-1)
    temperature_horizontal = 0.04  # Decreased temperature for more sensitivity
    horizontal_displacement_penalty = torch.exp(-horizontal_displacement / temperature_horizontal)  # Penalize large movements

    # Compute leg forces usage penalty (encourage minimal contact force)
    contact_force_usage = torch.sum(torch.abs(right_leg_contact_force) + torch.abs(left_leg_contact_force), dim=-1)
    temperature_force = 0.5  # Adjusted to penalize contact force usage
    contact_force_usage_penalty = torch.exp(-contact_force_usage / temperature_force)  # Penalize high contact force usage

    # Compute joint usage reward (encourages proper leg extension and contraction)
    leg_joints_indices = torch.tensor([6, 7, 8, 12, 13, 14], device=device)  # Indices of leg joints 
    leg_joint_usage = torch.mean(torch.abs(dof_pos[:, leg_joints_indices]), dim=-1)
    temperature_joints = 0.02  # Fine-tuned for joint usage sensitivity
    leg_joint_usage_reward = torch.exp(-torch.abs(leg_joint_usage) / temperature_joints)  # Encourage movements from the neutral position

    # Sum all rewards and penalties
    total_reward = vertical_movement_reward + horizontal_displacement_penalty + contact_force_usage_penalty + leg_joint_usage_reward

    # Create a dictionary for individual reward components
    reward_components = {
        'vertical_movement_reward': vertical_movement_reward,
        'horizontal_displacement_penalty': horizontal_displacement_penalty,
        'contact_force_usage_penalty': contact_force_usage_penalty,
        'leg_joint_usage_reward': leg_joint_usage_reward
    }

    return total_reward, reward_components
\end{lstlisting}
\end{tcolorbox}

%% file: iclr2026_conference.bib
@article{yao2023tree,
  title={Tree of thoughts: Deliberate problem solving with large language models},
  author={Yao, Shunyu and Yu, Dian and Zhao, Jeffrey and Shafran, Izhak and Griffiths, Tom and Cao, Yuan and Narasimhan, Karthik},
  journal={Advances in neural information processing systems},
  volume={36},
  pages={11809--11822},
  year={2023}
}

@article{xie2024monte,
  title={Monte carlo tree search boosts reasoning via iterative preference learning},
  author={Xie, Yuxi and Goyal, Anirudh and Zheng, Wenyue and Kan, Min-Yen and Lillicrap, Timothy P and Kawaguchi, Kenji and Shieh, Michael},
  journal={arXiv preprint arXiv:2405.00451},
  year={2024}
}

@article{ning2023skeleton,
  title={Skeleton-of-thought: Prompting llms for efficient parallel generation},
  author={Ning, Xuefei and Lin, Zinan and Zhou, Zixuan and Wang, Zifu and Yang, Huazhong and Wang, Yu},
  journal={arXiv preprint arXiv:2307.15337},
  year={2023}
}

@article{zhang2024rest,
  title={Rest-mcts*: Llm self-training via process reward guided tree search},
  author={Zhang, Dan and Zhoubian, Sining and Hu, Ziniu and Yue, Yisong and Dong, Yuxiao and Tang, Jie},
  journal={Advances in Neural Information Processing Systems},
  volume={37},
  pages={64735--64772},
  year={2024}
}

@misc{schulman2017proximalpolicyoptimizationalgorithms,
      title={Proximal Policy Optimization Algorithms}, 
      author={John Schulman and Filip Wolski and Prafulla Dhariwal and Alec Radford and Oleg Klimov},
      year={2017},
      eprint={1707.06347},
      archivePrefix={arXiv},
      primaryClass={cs.LG},
      url={https://arxiv.org/abs/1707.06347}, 
}

@misc{makoviychuk2021isaacgymhighperformance,
      title={Isaac Gym: High Performance GPU-Based Physics Simulation For Robot Learning}, 
      author={Viktor Makoviychuk and Lukasz Wawrzyniak and Yunrong Guo and Michelle Lu and Kier Storey and Miles Macklin and David Hoeller and Nikita Rudin and Arthur Allshire and Ankur Handa and Gavriel State},
      year={2021},
      eprint={2108.10470},
      archivePrefix={arXiv},
      primaryClass={cs.RO},
      url={https://arxiv.org/abs/2108.10470}, 
}

@article{brown2020language,
  title={Language models are few-shot learners},
  author={Brown, Tom and Mann, Benjamin and Ryder, Nick and Subbiah, Melanie and Kaplan, Jared D and Dhariwal, Prafulla and Neelakantan, Arvind and Shyam, Pranav and Sastry, Girish and Askell, Amanda and others},
  journal={Advances in neural information processing systems},
  volume={33},
  pages={1877--1901},
  year={2020}
}

@article{min2021metaicl,
  title={Metaicl: Learning to learn in context},
  author={Min, Sewon and Lewis, Mike and Zettlemoyer, Luke and Hajishirzi, Hannaneh},
  journal={arXiv preprint arXiv:2110.15943},
  year={2021}
}

@misc{wei2023chain,
      title={Chain-of-Thought Prompting Elicits Reasoning in Large Language Models}, 
      author={Jason Wei and Xuezhi Wang and Dale Schuurmans and Maarten Bosma and Brian Ichter and Fei Xia and Ed Chi and Quoc Le and Denny Zhou},
      year={2023},
      eprint={2201.11903},
      archivePrefix={arXiv},
      primaryClass={cs.CL},
      url={https://arxiv.org/abs/2201.11903}, 
}

@misc{wang2023self,
      title={Self-Consistency Improves Chain of Thought Reasoning in Language Models}, 
      author={Xuezhi Wang and Jason Wei and Dale Schuurmans and Quoc Le and Ed Chi and Sharan Narang and Aakanksha Chowdhery and Denny Zhou},
      year={2023},
      eprint={2203.11171},
      archivePrefix={arXiv},
      primaryClass={cs.CL},
      url={https://arxiv.org/abs/2203.11171}, 
}

@article{guo2025deepseek,
  title={Deepseek-r1: Incentivizing reasoning capability in llms via reinforcement learning},
  author={Guo, Daya and Yang, Dejian and Zhang, Haowei and Song, Junxiao and Zhang, Ruoyu and Xu, Runxin and Zhu, Qihao and Ma, Shirong and Wang, Peiyi and Bi, Xiao and others},
  journal={arXiv preprint arXiv:2501.12948},
  year={2025}
}

@article{shinn2023reflexion,
  title={Reflexion: Language agents with verbal reinforcement learning},
  author={Shinn, Noah and Cassano, Federico and Gopinath, Ashwin and Narasimhan, Karthik and Yao, Shunyu},
  journal={Advances in Neural Information Processing Systems},
  volume={36},
  pages={8634--8652},
  year={2023}
}

@article{madaan2023self,
  title={Self-refine: Iterative refinement with self-feedback},
  author={Madaan, Aman and Tandon, Niket and Gupta, Prakhar and Hallinan, Skyler and Gao, Luyu and Wiegreffe, Sarah and Alon, Uri and Dziri, Nouha and Prabhumoye, Shrimai and Yang, Yiming and others},
  journal={Advances in Neural Information Processing Systems},
  volume={36},
  pages={46534--46594},
  year={2023}
}

@misc{IMO2025,
  author       = {{International Mathematical Olympiad}},
  title        = {},
  howpublished = {\url{https://www.imo-official.org/}},
  year         = {2025}
}

@misc{CPHO2022,
  author       = {{Chinese Physics Olympiad}},
  title        = {},
  howpublished = {\url{http://cpho.pku.edu.cn/}},
  year         = {2022}
}

@misc{IPHO,
  author       = {{International Physics Olympiad}},
  title        = {},
  howpublished = {\url{https://www.ipho-new.org/}},
  year         = {}
}

@misc{IOI2025,
  author       = {{International Olympiad in Informatics}},
  title        = {},
  howpublished = {\url{https://ioinformatics.org/}},
  year         = {2025}
}

@misc{comanici2025gemini25pushingfrontier,
      title={Gemini 2.5: Pushing the Frontier with Advanced Reasoning, Multimodality, Long Context, and Next Generation Agentic Capabilities}, 
      author={Gheorghe Comanici and Eric Bieber and Mike Schaekermann and Ice Pasupat and Noveen Sachdeva and Inderjit Dhillon and Marcel Blistein and Ori Ram and Dan Zhang and others},
      year={2025},
      eprint={2507.06261},
      archivePrefix={arXiv},
      primaryClass={cs.CL},
      url={https://arxiv.org/abs/2507.06261}, 
}

@misc{kaplan2020scalinglawsneurallanguage,
      title={Scaling Laws for Neural Language Models}, 
      author={Jared Kaplan and Sam McCandlish and Tom Henighan and Tom B. Brown and Benjamin Chess and Rewon Child and Scott Gray and Alec Radford and Jeffrey Wu and Dario Amodei},
      year={2020},
      eprint={2001.08361},
      archivePrefix={arXiv},
      primaryClass={cs.LG},
      url={https://arxiv.org/abs/2001.08361}, 
}

@misc{rae2022scalinglanguagemodelsmethods,
      title={Scaling Language Models: Methods, Analysis \& Insights from Training Gopher}, 
      author={Jack W. Rae and Sebastian Borgeaud and Trevor Cai and Katie Millican and Jordan Hoffmann and Francis Song and John Aslanides and others},
      year={2022},
      eprint={2112.11446},
      archivePrefix={arXiv},
      primaryClass={cs.CL},
      url={https://arxiv.org/abs/2112.11446}, 
}

@misc{openai2024o1,
  author = {OpenAI},
  title = {Learning to reason with LLMs},
  year = {2024},
  howpublished = {\url{https://openai.com/index/learning-to-reason-with-llms}},
}

@misc{shi2025explainingcontextlengthscaling,
      title={Explaining Context Length Scaling and Bounds for Language Models}, 
      author={Jingzhe Shi and Qinwei Ma and Hongyi Liu and Hang Zhao and Jeng-Neng Hwang and Lei Li},
      year={2025},
      eprint={2502.01481},
      archivePrefix={arXiv},
      primaryClass={cs.LG},
      url={https://arxiv.org/abs/2502.01481}, 
}

@misc{huang2025gemini25procapable,
      title={Gemini 2.5 Pro Capable of Winning Gold at IMO 2025}, 
      author={Yichen Huang and Lin F. Yang},
      year={2025},
      eprint={2507.15855},
      archivePrefix={arXiv},
      primaryClass={cs.AI},
      url={https://arxiv.org/abs/2507.15855}, 
}

@misc{hoffmann2022trainingcomputeoptimallargelanguage,
      title={Training Compute-Optimal Large Language Models}, 
      author={Jordan Hoffmann and Sebastian Borgeaud and Arthur Mensch and Elena Buchatskaya and Trevor Cai and Eliza Rutherford and Diego de Las Casas and Lisa Anne Hendricks and Johannes Welbl and Aidan Clark and Tom Hennigan and Eric Noland and Katie Millican and George van den Driessche and Bogdan Damoc and Aurelia Guy and Simon Osindero and Karen Simonyan and Erich Elsen and Jack W. Rae and Oriol Vinyals and Laurent Sifre},
      year={2022},
      eprint={2203.15556},
      archivePrefix={arXiv},
      primaryClass={cs.CL},
      url={https://arxiv.org/abs/2203.15556}, 
}

@inproceedings{
aggarwal2025l,
title={L1: Controlling How Long A Reasoning Model Thinks With Reinforcement Learning},
author={Pranjal Aggarwal and Sean Welleck},
booktitle={Second Conference on Language Modeling},
year={2025},
url={https://openreview.net/forum?id=4jdIxXBNve}
}

@misc{muennighoff2025s1simpletesttimescaling,
      title={s1: Simple test-time scaling}, 
      author={Niklas Muennighoff and Zitong Yang and Weijia Shi and Xiang Lisa Li and Li Fei-Fei and Hannaneh Hajishirzi and Luke Zettlemoyer and Percy Liang and Emmanuel Candès and Tatsunori Hashimoto},
      year={2025},
      eprint={2501.19393},
      archivePrefix={arXiv},
      primaryClass={cs.CL},
      url={https://arxiv.org/abs/2501.19393}, 
}

@misc{snell2024scalingllmtesttimecompute,
      title={Scaling LLM Test-Time Compute Optimally can be More Effective than Scaling Model Parameters}, 
      author={Charlie Snell and Jaehoon Lee and Kelvin Xu and Aviral Kumar},
      year={2024},
      eprint={2408.03314},
      archivePrefix={arXiv},
      primaryClass={cs.LG},
      url={https://arxiv.org/abs/2408.03314}, 
}

@misc{wu2025inferencescalinglawsempirical,
      title={Inference Scaling Laws: An Empirical Analysis of Compute-Optimal Inference for Problem-Solving with Language Models}, 
      author={Yangzhen Wu and Zhiqing Sun and Shanda Li and Sean Welleck and Yiming Yang},
      year={2025},
      eprint={2408.00724},
      archivePrefix={arXiv},
      primaryClass={cs.AI},
      url={https://arxiv.org/abs/2408.00724}, 
}
